\documentclass[journal]{IEEEtran}
\usepackage{cite}
\usepackage{amsmath,amssymb,amsfonts,amsthm,mathtools,bm,bbm}
\usepackage{algorithmic}
\usepackage{graphicx}
\usepackage{textcomp}
\usepackage{xcolor}
\usepackage{microtype}

\usepackage[
  colorlinks=false,
  pdfborder={0 0 1},
  citebordercolor={0 1 0},   % bibliography citations: green
  linkbordercolor={1 0 0},   % internal cross-references: red
  urlbordercolor={0 1 1}
]{hyperref}

\usepackage{orcidlink}

\theoremstyle{plain}
\newtheorem{theorem}{Theorem}
\newtheorem{lemma}{Lemma}
\newtheorem{corollary}{Corollary}

\theoremstyle{definition}

\newtheorem{remark}{Remark}

\DeclareMathOperator{\Tr}{Tr}
\DeclareMathOperator{\diag}{diag}

\begin{document}

\title{PAC-Bayesian Adversarially Robust Generalization for Message Passing Graph Neural Networks: \\A Sensitivity Analysis}

\author{
Ziling~Liang,
Xinping~Yi,~\IEEEmembership{Member,~IEEE},
Qingsong~Wen,~\IEEEmembership{Senior Member,~IEEE},
and~Shi~Jin,~\IEEEmembership{Fellow,~IEEE}
\thanks{Z. Liang, X. Yi, and S. Jin are with the School of Information
Science and Engineering, Southeast University, Nanjing 210096, China.
E-mail: \{zlliang, xyi, jinshi\}@seu.edu.cn.}
\thanks{Q. Wen is with Squirrel Ai Learning, Bellevue, WA, 98004 USA. Email: qingsongedu@gmail.com.}
}

\maketitle

\begin{abstract}
Whilst the vulnerability of graph neural networks (GNNs) to adversarial attacks poses a critical threat to graph representation learning, the understanding of the robust generalization behavior remains a fundamental challenge in the adversarial setting. Recently, PAC-Bayesian margin-based generalization analysis substantially advances this line of research by providing a flexible and data-dependent analytical framework. However, existing robust analyses often rely on isotropic Gaussian posteriors and control weight perturbations in the full parameter space, which limits the ability to capture heterogeneous parameter sensitivity yet hinges on hidden-width-dependent complexity terms, resulting in not-tight-enough generalization bounds.
In this paper, we extend a recently proposed sensitivity-aware PAC-Bayesian framework from deep neural networks to message passing GNNs (MPGNNs) and derive a tighter robust generalization bound in the adversarial setting. Specifically, we first quantify how sensitive the perturbations across different parameter blocks are to the network outputs by deriving the output Jacobians with respect to the weight parameters. Exploiting the fact that these Jacobian matrices have rank at most \(K\) in \(K\)-class graph classification, we then construct Jacobian-aligned sensitivity matrices and use anisotropic Gaussian posteriors with optimized covariances to upper bound the KL divergence in a tight way.
Notably, by refining the spectral-norm dependence on the learned weights and reducing the leading dimension factor from hidden-width-dependent terms to the number of classes \(K\), our analysis yields much tighter robust generalization guarantees for MPGNNs,  thereby guiding their designs to enhance adversarial robustness.
\end{abstract}
\section{Introduction}

Graph neural networks (GNNs) have achieved remarkable performance in a
wide range of graph-related tasks, emerging as a powerful learning
paradigm for relational and structured data~\cite{GNN-Model,Graph-Networks,GCN}. They have been successfully applied to social
networks~\cite{GraphSAGE}, biological modeling~\cite{Decagon},
recommendation systems~\cite{PinSage}, and wireless networks~\cite{Wireless-GNN}, to name just a few. The key mechanism of GNNs is to propagate information over the graph through local computations shared across nodes and edges, which enables the learned representations to capture both feature information and graph-structural dependencies.

Despite their empirical success, the reliability of GNNs under adversarial perturbations remains a critical concern. Similar to the vulnerability of convolutional neural networks (CNNs) to adversarial examples in image classification tasks~\cite{Szegedy14,Goodfellow15}, where certain imperceptible perturbations of input images can mislead the well-trained models to make wrong classification in the inference phase, the graph learning tasks using GNNs can also suffer significantly from adversarial perturbations on node features or graph
structures~\cite{Nettack,Metattack,AdvExamples-Graph,GraphAttack-BO,Revisit-GraphAttack}.
The consequence of adversarial perturbations is even more severe for GNNs,
as local perturbations can be propagated and amplified through message passing, thereby affecting graph-level predictions. It, therefore, provides a stringent test of whether learned representations are stable
to both attribute-level and topology-level variations. As such, there is a growing line of work in the literature to investigate adversarial attacks and defenses (e.g., adversarial training) for
GNNs~\cite{GNNGuard,LowRankDefense,CausalGNNDefense-TIFS,ATGSE-TIFS} to probe and defend against the potential adversarial vulnerability.
However, these algorithmic and/or empirical studies do not convincingly explain whether the
adversarial robustness observed on training graphs can generalize to unseen graph
samples. Such a gap makes it essential to analyze adversarially robust generalization for understanding the reliability of GNNs.

From a statistical learning perspective, the weight sharing and topology-dependent message passing pose fundamental challenges for classical generalization analyses. To address these challenges, several theoretical frameworks have been explored for GNNs. The VC-dimension analyses provide capacity measures for GNNs, but the resulting bounds often scale unfavorably with model and graph size~\cite{Scarselli-VC,Esser-LearningTheory-GNN}. 
More refined data-dependent analyses use Rademacher complexity to study message passing GNNs (MPGNNs) and graph convolutional networks (GCNs), relating the bounds to local computation trees, aggregation neighborhoods, and graph-dependent quantities~\cite{Generalization-Limits-GNN,Esser-LearningTheory-GNN,Lv-Rademacher-GCN}.
The algorithmic stability analyses provide another perspective by connecting generalization to the stability of the training
algorithm and the spectral behavior of graph filters~\cite{Stability-GCN,Zhou-Generalization-Error-GCN}. 
The neural tangent kernel (NTK) methods further characterize the infinite-width regime of GNNs trained by gradient descent through graph NTKs~\cite{GNTK}. 
These studies indicate that generalization performance for GNNs is shaped by graph statistics, aggregation dynamics, feature propagation, and the norm complexity of the learned weights. Nevertheless, most existing bounds are developed for standard non-adversarial settings and analyze these factors separately.

Among existing approaches, PAC-Bayesian provides a flexible and
data-dependent framework for deriving high-probability generalization
bounds~\cite{McAllester03,Langford03,Catoni07,PACBayes-DataPriors,
PACBayes-Nonvacuous}. In deep learning, norm-based and margin-based analyses connect generalization to the stability of predictions under random weight perturbations, leading to bounds governed by norm-based complexity measures of the learned weights~\cite{Spectral-Margin-Bounds,PACBayes-SpectralNorm}.
The norm-based viewpoint has also been extended to graph-structured models. In particular, PAC-Bayesian bounds for GCNs and MPGNNs show that graph-dependent quantities and spectral norms of the
learned weights play a central role in controlling GNN generalization~\cite{Liao-PACBayes-GNN}. For GCNs, these bounds can be viewed as graph
extensions of spectrally normalized margin bounds for feedforward and
convolutional networks, while for MPGNNs they require a separate treatment because the input-injection and message-passing terms make the architecture non-homogeneous in the layer weights. Subsequent PAC-Bayesian analyses further refine this line of work by measuring graph dependence through the spectral norm of the graph diffusion operator, rather than solely through node degrees~\cite{Graph-Diffusion-PACBayes,yi2026topology}.
In adversarial settings, PAC-Bayesian analysis has been extended from
standard margin losses to robust margin losses~\cite{PACBayes-Adversarial,PACBayes-Adversarial-Spectral}. This viewpoint is particularly useful because the worst-case adversarial example depend on the model weights, making direct output comparison under weight perturbations difficult. For GNNs, robust PAC-Bayesian bounds evaluate margins over adversarially perturbed graph samples and show that the resulting robust bounds are governed by the norms of the learned weights and the adversarial perturbation budget~\cite{Sun-PACBayes-Robust}.

Despite the growing body of PAC-Bayesian generalization bounds for GNNs,
existing analyses still exhibit several structural limitations.
First, most prior framework adopt isotropic Gaussian posteriors for analytical tractability, implicitly assuming homogeneous
uncertainty across all parameters; and the perturbation bound is governed by matrix concentration inequalities which are loose for deep or wide architectures.
Second, existing generalization analyses bound perturbations in the full parameter space, without exploiting the low-rank structure of the output sensitivity in graph classification tasks, leading the
complexity term depend on the maximum hidden dimension rather than the effective output dimension.
Third, existing PAC-Bayesian bounds usually analyze GCNs and MPGNNs
through separate perturbation arguments. Although GCN-type models can be
viewed as special cases of the MPGNN formulation, the connection is
not fully used to obtain a unified bound. Moreover, due to the
non-homogeneous structure of MPGNNs, existing analyses still use relatively coarse spectral factors for the learned weights, leaving room to tighten both the spectral dependence and the dimensional dependence of the bound.

Motivated by the above limitations, we develop a novel sensitivity-aware PAC-Bayesian analytical framework for adversarially robust generalization of MPGNNs, with the main contributions summarized as follows.
\begin{itemize}
\item \textbf{Unified robust PAC-Bayesian framework for MPGNNs.}
We extend the recently proposed sensitivity-aware anisotropic PAC-Bayesian framework in~\cite{Towards-Unified} from deep neural networks in the standard setting to graph-structured models in the adversarial setting, and develop a unified robust margin-based generalization analysis for MPGNNs under adversarial perturbations.
\item \textbf{Low-rank Jacobian-based perturbation analysis.}
We derive the blockwise output Jacobians of MPGNNs and show that they
have a low-rank structure determined by the output dimension, which allows us to design low-rank sensitivity matrices. Together with anisotropic Gaussian posteriors, these matrices focus the perturbation analysis on the parameter directions that affect the output margin.
\item \textbf{Refined spectral and dimensional dependencies.}
We tighten the robust PAC-Bayesian bound for MPGNNs in both spectral and
dimensional dependence, i.e., the spectral scale of the learned weights is
tightened, and the leading dimension factor is reduced from a
hidden-width-dependent term to the number of classes, and show that the obtained bound also
recovers GCNs as special cases.
\end{itemize}

The rest of this paper is organized as follows.
Section~\ref{sec:prelim} introduces the problem setting and necessary
preliminaries.
Section~\ref{sec:main-bound} develops the main robust PAC-Bayesian analysis
for MPGNNs and its GCN specialization.
Section~\ref{sec:conclusion} concludes the paper.
Detailed proofs are provided in the Appendix.
\section{Preliminaries}
\label{sec:prelim}
\textbf{Notation.} We collect the notation used throughout the paper. For a positive integer \(k\), let \([k]\triangleq\{1,\ldots,k\}\). Scalars, vectors, matrices, and sets are denoted by \(a\), \(\bm a\),
\(\bm A\), and \(\mathcal A\), respectively. The \(i\)-th entry of
\(\bm a\), the \(i\)-th row of \(\bm A\), and the \((i,j)\)-th entry of
\(\bm A\) are denoted by \(\bm a_i\), \(\bm A_{i,:}\), and
\(\bm A_{ij}\), respectively. We use \(\bm 1_n\), \(\bm 0\), and \(\bm I_n\) for the all-one vector, zero vector or matrix, and identity matrix, with dimensions clear from the context. The vectorization operator is denoted by \(\mathrm{vec}(\cdot)\), \(\diag(\cdot)\) denotes a diagonal matrix, and \(\mathbbm{1}\{\cdot\}\) denotes the indicator function. For a vector \(\bm a\), \(\|\bm a\|_2\) and \(\|\bm a\|_\infty\) denote the Euclidean and maximum norms. For a matrix \(\bm A\), \(\|\bm A\|_2\), \(\|\bm A\|_F\), and \(\|\bm A\|_{2,\infty}\triangleq\max_i\|\bm A_{i,:}\|_2\) denote the spectral norm, Frobenius norm, and maximum row Euclidean norm, respectively. We write \(\Tr(\bm A)\), \(\det(\bm A)\), and \(\log\det(\bm A)\) for the trace, determinant, and log-determinant of a square matrix. For symmetric matrices, \(\bm A\succeq\bm 0\) means that \(\bm A\) is positive semidefinite, and \(\bm A\succeq\bm B\) means \(\bm A-\bm B\succeq\bm 0\). The Kronecker product is denoted by \(\otimes\), and we use \(\mathrm{vec}(\bm A\bm X\bm B)=(\bm B^{T}\otimes\bm A)\mathrm{vec}(\bm X)\) and \(\|\bm A\otimes\bm B\|_2=\|\bm A\|_2\|\bm B\|_2\). Finally, \(\mathcal N(\bm\mu,\bm\Sigma)\) denotes a Gaussian distribution with mean vector \(\bm\mu\) and covariance matrix \(\bm\Sigma\succeq \bm 0\), \(\operatorname{blkdiag}(\cdot)\) denotes a block-diagonal matrix, and \(D_{\mathrm{KL}}(Q\|P)\) denotes the Kullback--Leibler divergence between the distributions \(Q\) and \(P\).

\subsection{Graph Neural Networks (GNNs)}

We consider the $K$-class graph classification task for simple and
undirected graphs. The input space
$\mathcal{Z}$ consists of graph samples
$z=(G,y)\in\mathcal{Z}$,
where \(G=(\bm X,\bm A_G)\),
$\bm X\in\mathbb{R}^{n\times h_0}$
denotes the node feature matrix over $n$ nodes and satisfies
$\|\bm X\|_{2,\infty}=\max_i \|\bm X_{i,:}\|_2\le B$,
$\bm A_G\in\{0,1\}^{n\times n}$ is the adjacency matrix, and
$y\in\{1,\ldots,K\}$ denotes the output label. We denote
$y=\arg\max_{y} f_{\bm w}(G)[y]$, where
$f_{\bm w}\in\mathcal{H}:\mathcal{X}\times\mathcal{G}\to\mathbb{R}^K$
is a function specified by a parameterized learning model with model
parameters $\bm w$. The spaces $\mathcal{Z}$, $\mathcal{X}$,
$\mathcal{G}$, and $\mathcal{H}$ denote the input sample space, node
feature space, graph space, and hypothesis class space, respectively.
A dataset $S=\{z_1,\ldots,z_m\}$ with $m$ training samples is drawn
independently and identically from an unknown distribution
$\mathcal{D}$ is given to learn the model parameters $\bm w$.

\textbf{MPGNNs.}
Message passing GNNs cover a wide range of graph learning architectures based on iterative neighborhood aggregation and node-wise transformations~\cite{Dai16-Structured-Data,Generalization-Limits-GNN}.
We study the following MPGNN formulation for graph classification, with each layer containing an input-injection block and a message-passing block, i.e.,
\begin{align}
\bm{H}_j
&=
\phi_j\!\Bigl(
\bm{X}\bm{U}_j
+
\rho_j\!\bigl(
\bm{P}_G \psi_j(\bm{H}_{j-1}) \bm{W}_j
\bigr)
\Bigr),
\quad j\in[l-1], \notag\\
\bm{H}_l
&=
f_{\bm{w}}(G)
=
\frac{1}{n}\bm{1}_n^\top \bm{H}_{l-1}\bm{W}_l,
\quad
\bm{H}_0=\bm{0},
\label{eq:mpgnn-model}
\end{align}
where \(\bm X\in\mathbb R^{n\times d}\) denotes the node feature matrix,
and \(\bm H_j\in\mathbb R^{n\times h_j}\), \(j\in[l-1]\), denotes the
node representation at the \(j\)-th hidden layer, the final readout is
\(\bm H_l=f_{\bm w}(G)\in\mathbb R^{1\times K}\). For each \(j\in[l-1]\),
\(\bm U_j\in\mathbb R^{d\times h_j}\) is the input-injection weight
matrix and \(\bm W_j\in\mathbb R^{h_{j-1}\times h_j}\) is the
message-passing weight matrix, while
\(\bm W_l\in\mathbb R^{h_{l-1}\times K}\) is the readout weight matrix.
We also write \(h=\max_{j\in[l-1]}h_j\) for the maximum hidden width.

The propagation operator in \eqref{eq:mpgnn-model} is taken as the
normalized adjacency matrix with self-loops,
\begin{equation}
\label{eq:PG-normalized-adj}
\bm P_G
=
\bm D_G^{-\frac{1}{2}}
(\bm I+\bm A_G)
\bm D_G^{-\frac{1}{2}},
\end{equation}
where
\(\bm D_G=\operatorname{diag}(1+D_1,\ldots,1+D_n)\), and \(D_i\) is the
degree of node \(i\) in \(G\). This normalization gives
\(\|\bm P_G\|_2\le 1\).
For later use, we decompose the \(j\)-th hidden-layer update into the
pre-activation term \(\bm E_j=\bm X\bm U_j+\rho_j(\bm F_j)\) and the
message-passing term \(\bm F_j=\bm P_G\psi_j(\bm H_{j-1})\bm W_j\), for
\(j\in[l-1]\).

The trainable parameters are collected from the matrices
\(\{\bm U_j\}_{j=1}^{l-1}\) and \(\{\bm W_j\}_{j=1}^{l}\). More precisely,
let
\[
\bm w=
\bigl(
\bm w_{U_1}^{T},\ldots,\bm w_{U_{l-1}}^{T},
\bm w_{W_1}^{T},\ldots,\bm w_{W_l}^{T}
\bigr)^{T},
\]
where \(\bm w_{U_j}=\operatorname{vec}(\bm U_j)\) and
\(\bm w_{W_j}=\operatorname{vec}(\bm W_j)\). The weight perturbation is
written in the same form as
\[
\bm u=
\bigl(
\bm u_{U_1}^{T},\ldots,\bm u_{U_{l-1}}^{T},
\bm u_{W_1}^{T},\ldots,\bm u_{W_l}^{T}
\bigr)^{T},
\]
where \(\bm u_{U_j}=\operatorname{vec}(\Delta\bm U_j)\) and
\(\bm u_{W_j}=\operatorname{vec}(\Delta\bm W_j)\). As such, the perturbed model can
therefore be parameterized by \(\bm w+\bm u\).

Throughout the analysis, we assume that the node features are bounded as
\(\|\bm X\|_{2,\infty}\le B\). For the layer weights, define
\[
M_1=\max_{j\in[l-1]}\|\bm U_j\|_2,
\qquad
M_2=\max_{j\in[l]}\|\bm W_j\|_2 .
\]
Thus \(\|\bm U_j\|_2\le M_1\) for \(j\in[l-1]\) and
\(\|\bm W_j\|_2\le M_2\) for \(j\in[l]\). We take \(M_1,M_2\ge 1\), which is consistent with the spectral scale of trained model parameters and will simplify the subsequent upper-bound comparisons. The mappings \(\phi_j\), \(\rho_j\), and \(\psi_j\) are applied element-wise and are assumed to be \(L\)-Lipschitz, satisfying \(\phi_j(\bm 0)=\rho_j(\bm 0)=\psi_j(\bm 0)=\bm 0\).

The graph convolutional network (GCN) is obtained from
\eqref{eq:mpgnn-model} by setting \(\bm U_j=\bm 0\) for all $j$, taking
\(\rho_j\) and \(\psi_j\) as identity mappings, choosing \(\phi_j\) as
ReLU, and using \(\bm H_0=\bm X\). Thus the \(l\)-layer GCN can be rewritten as
\[
\begin{aligned}
\bm H_j
&=
\phi_j\!\left(\bm P_G\bm H_{j-1}\bm W_j\right),
\qquad j\in[l-1],
\\
\bm H_l
&=
f_{\bm w}(G)
=
\frac{1}{n}\bm 1_n^{T}\bm H_{l-1}\bm W_l .
\end{aligned}
\]

\subsection{Background of PAC-Bayes Analysis}

\textbf{Margin loss.}
In \(K\)-class graph classification, margin-based losses quantify the
confidence gap between the true class and its competitors and serve as the risk functions in the PAC-Bayesian analysis. We first define the margin with respect to the true label \(y\in[K]\) as
\begin{equation}
\label{eq:margin-operator-y}
M\!\left(f_{\bm w}(G),y\right)
=
f_{\bm w}(G)_y
-
\max_{b\ne y} f_{\bm w}(G)_b .
\end{equation}
For the adversarial perturbation analysis, it is also useful to compare
the scores of two classes directly. Thus, for any \(a,b\in[K]\), we define the
pairwise margin as
\begin{equation}
\label{eq:margin-operator-ij}
M\!\left(f_{\bm w}(G),a,b\right)
=
f_{\bm w}(G)_a
-
f_{\bm w}(G)_b .
\end{equation}
The condition \(M(f_{\bm w}(G),y)>0\) indicates correct classification.
For any margin level \(\gamma>0\), we define the population margin loss
and its empirical counterpart as
\begin{equation}
\label{eq:standard-margin-losses}
\begin{aligned}
L_{\mathcal D,\gamma}(f_{\bm w})
&=
\mathbb P_{(G,y)\sim\mathcal D}
\Bigl\{
M\!\left(f_{\bm w}(G),y\right)\le\gamma
\Bigr\},
\\
\hat L_{S,\gamma}(f_{\bm w})
&=
\frac{1}{m}
\sum_{(G,y)\in S}
\mathbbm 1\!\Bigl\{
M\!\left(f_{\bm w}(G),y\right)\le\gamma
\Bigr\}.
\end{aligned}
\end{equation}
The standard generalization gap is
\(L_{\mathcal D,0}(f_{\bm w})-\hat L_{S,\gamma}(f_{\bm w})\).

We now incorporate adversarial perturbations by replacing each graph with its worst admissible perturbation. Let \(\delta_{\bm w}(G)\) denote the attack set associated with \(G\) and \(f_{\bm w}\). Under
\(\epsilon\)-attack, every \(G'=(\bm X',\bm A_G')\in\delta_{\bm w}(G)\) satisfies \(\|\bm X'-\bm X\|_{2,\infty}\le\epsilon\), while \(\bm A_G'\) may be any adjacency matrix. As such, the robust margins are defined as
\begin{equation}
\label{eq:robust-margin-operators}
\begin{aligned}
RM\!\left(f_{\bm w}(G),y\right)
&=
\inf_{G'\in\delta_{\bm w}(G)}
M\!\left(f_{\bm w}(G'),y\right),
\\
RM\!\left(f_{\bm w}(G),a,b\right)
&=
\inf_{G'\in\delta_{\bm w}(G)}
M\!\left(f_{\bm w}(G'),a,b\right).
\end{aligned}
\end{equation}
Let \(G_{\bm w}^{*}\in\arg\min_{G'\in\delta_{\bm w}(G)}
M(f_{\bm w}(G'),y)\) denote an adversarial graph that attains the
worst-case margin. Replacing the
standard margin by the robust margin in \eqref{eq:standard-margin-losses}
gives
\begin{equation}
\label{eq:robust-margin-losses}
\begin{aligned}
R_{\mathcal D,\gamma}(f_{\bm w})
&=
\mathbb P_{(G,y)\sim\mathcal D}
\Bigl\{
RM\!\left(f_{\bm w}(G),y\right)\le\gamma
\Bigr\},
\\
\hat R_{S,\gamma}(f_{\bm w})
&=
\frac{1}{m}
\sum_{(G,y)\in S}
\mathbbm 1\!\Bigl\{
RM\!\left(f_{\bm w}(G),y\right)\le\gamma
\Bigr\}.
\end{aligned}
\end{equation}
The robust generalization gap is
\(R_{\mathcal D,0}(f_{\bm w})-\hat R_{S,\gamma}(f_{\bm w})\).

\textbf{PAC-Bayesian generalization bound~\cite{PACBayes-SpectralNorm}.}
Let \(f_{\bm w}:\mathcal X\to\mathbb R^K\) be a deterministic classifier
with parameters \(\bm w\), and let \(P\) be a prior distribution over the
parameters independent of the training data. For any fixed \(\bm w\), let
\(Q\) denote the distribution of the perturbed parameters
\(\bm w+\bm u\), where \(\bm u\) is a random weight perturbation. If the perturbation condition is satisfied, i.e.,
\begin{equation}
\label{eq:pacbayes-perturbation-condition}
\mathbb P_{\bm u}\!\left(
\max_{\bm x\in\mathcal X}
\bigl\|
f_{\bm w+\bm u}(\bm x)-f_{\bm w}(\bm x)
\bigr\|_\infty
<
\frac{\gamma}{4}
\right)
\ge
\frac{1}{2},
\end{equation}
then, for any \(\gamma,\delta>0\), with probability at least \(1-\delta\)
over an i.i.d. training set \(S\) of size \(m\) drawn from
\(\mathcal D\),  we have the generalization bound for any \(\bm w\), i.e.,
\begin{equation}
\label{eq:perturbation-pacbayes}
L_{\mathcal D,0}(f_{\bm w})
\le
\hat L_{S,\gamma}(f_{\bm w})
+
\sqrt{
\frac{
2D_{\mathrm{KL}}(Q\|P)
+
\log\frac{8m}{\delta}
}{
2(m-1)
}
}.
\end{equation}

\textbf{Robust PAC-Bayesian bound for GCN~\cite{Sun-PACBayes-Robust}.}
Under the same sampling and confidence setting as in \eqref{eq:perturbation-pacbayes}, the robust
bound for an \(l\)-layer GCN under \(\epsilon\)-attack gives
\begin{equation}
\label{eq:existing-robust-gcn-bound}
\begin{aligned}
R_{\mathcal D,0}(f_{\bm w})
&\le
\hat R_{S,\gamma}(f_{\bm w})
\\
&\hspace*{-1.4em}+
\mathcal O\!\left(
\sqrt{
\frac{
(B+\epsilon)^2l^2h\log(lh)\Phi(f_{\bm w})
+
\log\frac{ml}{\delta}
}{
\gamma^2m
}
}
\right)
\end{aligned}
\end{equation}
with spectral complexity
\(\Phi(\bm w)=\prod_{l=1}^{d}\|\bm W_l\|_2^2
\sum_{l=1}^{d}\frac{\|\bm W_l\|_F^2}{\|\bm W_l\|_2^2}\).

The key step behind the robust bound is to control the change of the
robust margin under the weight perturbation. Specifically,
for any perturbation \(\Delta\bm w=\operatorname{vec}(\{\Delta\bm W_k\}_{k=1}^{l})\) satisfying
\(\|\Delta\bm W_k\|_2\le \|\bm W_k\|_2/l\) for all \(k\in[l]\), and for
any \(a,b\in[K]\), the robust margin perturbation is bounded as
\begin{equation}
\label{eq:existing-robust-margin-perturbation}
\begin{aligned}
\MoveEqLeft\left|
RM(f_{\bm w+\Delta\bm w}(G),a,b)
-
RM(f_{\bm w}(G),a,b)
\right|
\\
&\le
2e(B+\epsilon)
\left(
\prod_{k=1}^{l}\|\bm W_k\|_2
\right)
\sum_{k=1}^{l}
\frac{\|\Delta\bm W_k\|_2}{\|\bm W_k\|_2}.
\end{aligned}
\end{equation}

\subsection{The Unified Framework}

\textbf{Anisotropic perturbations and sensitivity matrices.}
The unified PAC-Bayesian framework in~\cite{Towards-Unified} establishes a flexible analytical framework for generalization by introducing anisotropic Gaussian posteriors for weight perturbations and sensitivity matrices that measure how different perturbation directions affect the model output. The weight perturbation is partitioned into \(l\) parameter blocks,
\(\bm u=(\bm u_1^{T},\ldots,\bm u_l^{T})^{T}\). The framework considers
\(\bm u\sim\mathcal N(\bm 0,\sigma^2\bm R)\), where
\(\bm R=\operatorname{blkdiag}(\bm R_1,\ldots,\bm R_l)\). Each covariance
block \(\bm R_j\) specifies the posterior variance within the \(j\)-th
parameter block, so that the distribution of weight perturbations can adapt to direction-dependent sensitivities rather than assigning a common variance
to all directions. To quantify the effect of these perturbations, a
block-diagonal sensitivity matrix
\(\bm A=\operatorname{blkdiag}(\bm A_1,\ldots,\bm A_l)\) is introduced,
where \(\bm A_j\) controls the sensitivity w.r.t. \(\bm u_j\).

\textbf{Optimization Formula.}
For the Gaussian prior \(P=\mathcal N(\bm 0,\sigma^2\bm I)\) and the
anisotropic posterior \(Q=\mathcal N(\bm w,\sigma^2\bm R)\), the KL term
takes the form~\cite{Pardo-Divergence}
\[
D_{\mathrm{KL}}
=
\frac{1}{2}\sum_{j=1}^{l}
\left[
\tfrac{\|\bm W_j\|_F^2}{\sigma^2}
+
\Tr(\bm R_j)
-
\log\det\bm R_j
-
\dim(\bm R_j)
\right].
\]
The framework in~\cite{Towards-Unified} seeks sensitivity matrices and posterior
covariance blocks that satisfy the perturbation condition while minimizing
KL, yielding the following optimization
\begin{subequations}
\label{eq:yi-opt-problem}
\begin{align}
\min_{\sigma^2,\{\bm R_j\},\{\bm A_j\}}
\quad &
D_{\mathrm{KL}}(\sigma^2,\{\bm R_j\}),
\label{eq:yi-opt-problem-a}
\\[-0.2em]
\text{s.t.}\quad&
\mathbb P_{\bm u\sim\mathcal N(\bm 0,\sigma^2\bm R)}
\!\left[
\sum_{j=1}^{l}\|\bm A_j\bm u_j\|_2^2
<
\frac{\gamma^2}{16}
\right]
\ge
\frac{1}{2},
\label{eq:yi-opt-problem-b}
\\[-0.1em]
&\|f_{\bm w+\bm u}(\bm x)-f_{\bm w}(\bm x)\|_{\infty}^{2}
\le
\sum_{j=1}^{l}\|\bm A_j\bm u_j\|_2^2.
\label{eq:yi-opt-problem-c}
\end{align}
\end{subequations}
Here \eqref{eq:yi-opt-problem-a} minimizes the KL complexity of the
anisotropic posterior, \eqref{eq:yi-opt-problem-b} enforces the
probabilistic perturbation condition, and \eqref{eq:yi-opt-problem-c}
connects output stability to the quadratic sensitivity term. Notably, \(\bm A_j\) is to encode the blockwise sensitivity of the model output,
while \(\bm R_j\) determines how posterior variance is allocated across
different perturbation directions.

\textbf{Variance choice.}
The probabilistic constraint in \eqref{eq:yi-opt-problem-b} is handled by
concentration of Gaussian quadratic forms~\cite{Rudelson-HansonWright}.
With \(\kappa=1+2\ln 2+\sqrt{4\ln 2}\), it is sufficient to ensure, with
probability at least \(1/2\), that
\begin{equation}
\label{eq:yi-sigma-choice}
\sum_{j=1}^{l}\|\bm A_j\bm u_j\|_2^2
\le
\sigma^2\kappa
\sum_{j=1}^{l}
\Tr(\bm A_j\bm R_j\bm A_j^{T})
\le
\frac{\gamma^2}{16}.
\end{equation}
Thus, \(\sigma^2\) is chosen as the largest value allowed by
\eqref{eq:yi-sigma-choice}, so as to satisfy the perturbation condition
while minimizing the KL term in~\eqref{eq:yi-opt-problem-a}.

\textbf{Optimized posterior covariance.}
For fixed sensitivity matrices, minimizing the KL term in
\eqref{eq:yi-opt-problem-a} with respect to the covariance blocks gives
\begin{equation}
\label{eq:yi-R-opt}
\bm R_j^{*}
=
\left(
\bm I
+
\frac{16\kappa\|\bm w\|_2^2}{\gamma^2}
\bm A_j^{T}\bm A_j
\right)^{-1}.
\end{equation}
Thus, directions with larger sensitivity, as measured by
\(\bm A_j^{T}\bm A_j\), receive smaller posterior variance, while less
sensitive directions can retain larger variance.

In doing so, the unified framework reduces the derivation of PAC-Bayesian bounds to
three steps: 1) Construct sensitivity matrices that dominate the perturbation
effect; 2) Choose \(\sigma^2\) so that the perturbation condition holds, and
3) Evaluate the KL term using the optimized covariance \(\bm R_j^{*}\).
\section{Robust Generalization Bounds}
\label{sec:main-bound}
In this section, we extend the unified framework from deep neural networks to graph-structured models and from the standard setting to the adversarial setting. Specifically, we present a tighter adversarially robust generalization bound for MPGNNs, which also covers classical graph learning architectures such as GCNs as special cases.

In what follows, we start with some useful lemmas, followed by a modified version of the optimization problem in \eqref{eq:yi-opt-problem} dedicated to the adversarial settings. 

\begin{lemma}[Jacobian matrices with respect to \(\bm U_j\) and \(\bm W_j\)]
\label{lem:mpgnn-Jacobian-UW}
By defining \(\bm J_{U_j}=\frac{\partial f_{\bm w}(G)}{\partial \mathrm{vec}(\bm U_j)}\) and \(\bm J_{W_j}=\frac{\partial f_{\bm w}(G)}{\partial \mathrm{vec}(\bm W_j)}\) as the Jacobian matrices of the network output with respect to the input-injection weight matrix \(\bm U_j\) and the message-passing weight matrix \(\bm W_j\), we have
\begin{align}
\label{eq:mpgnn-JUj-main}
\bm J_{U_j}
&=
\bm M_j\bm B_j^{\phi}
\left(
\bm I_{h_j}\otimes \bm X
\right),
\quad
j\in[l-1],\\
\label{eq:mpgnn-JWj-main}
\bm J_{W_j}
&=
\bm M_j\bm B_j^{\phi}\bm B_j^{\rho}
\left(
\bm I_{h_j}\otimes \bm P_G\psi_j(\bm H_{j-1})
\right),
\quad
j\in[l]
\end{align}
where \(\bm M_j\triangleq\frac{\partial f_{\bm w}(G)}{\partial \mathrm{vec}(\bm H_j)}\) is the Jacobian matrix of the network output with respect to the graph embeddings at the \(j\)-th layer
\begin{equation}
\label{eq:mpgnn-Mj-main}
\bm M_j
=
\frac{1}{n}
\left(
\bm W_l^{T}\otimes \bm 1_n^{T}
\right)
\left(
\prod_{k=l-1}^{j+1}
\bm B_k^{\phi}\bm B_k^{\rho}
\left(
\bm W_k^{T}\otimes \bm P_G
\right)
\bm B_k^{\psi}
\right),
\end{equation}
with the diagonal matrices
\begin{equation}
\label{eq:mpgnn-Bvartheta-k-main}
\bm B_k^{\vartheta}
\triangleq
\begin{cases}
\diag\!\bigl(\mathrm{vec}(\phi_k'(\bm E_k))\bigr),
& \vartheta=\phi,\\[0.4em]
\diag\!\bigl(\mathrm{vec}(\rho_k'(\bm F_k))\bigr),
& \vartheta=\rho,\\[0.4em]
\diag\!\bigl(\mathrm{vec}(\psi_k'(\bm H_{k-1}))\bigr),
& \vartheta=\psi,
\end{cases}
\end{equation}
where \(\vartheta_k'\) denotes the derivative of the activation function
\(\vartheta\), with \(\vartheta\in\{\phi,\rho,\psi\}\).
\end{lemma}

\begin{proof}
See Appendix~\ref{app:proof-mpgnn-Jacobian-UW}.
\end{proof}

\begin{lemma}[Low-rank sensitivity matrices for MPGNN]
\label{lem:mpgnn-lr-sensitivity-domination}
For \(l>1\), \(B,L>0\), \(M_1,M_2\ge1\), define \(\tau=L^3M_2\), and let
\begin{equation}
\label{eq:mpgnn-beta-definition}
\beta
=
\begin{cases}
\left(
LM_1M_2\dfrac{\tau^{l-1}-1}{\tau-1}
\right)^{1/l},
& \tau\neq 1,\\[0.8em]
\sqrt{LM_1M_2},
& \tau=1.
\end{cases}
\end{equation}
The spectral norm bound for all Jacobian matrices is given by
\begin{equation}
\label{eq:mpgnn-common-jacobian-bound-main}
\max_{j\in[l]}
\left\{
\|\bm J_{U_j}\|_2,\,
\|\bm J_{W_j}\|_2
\right\}
\le
B s ,
\end{equation}
where \(s=\beta^l\) if \(\tau\neq1\), \(s=(l-1)\beta^2\) if
\(\tau=1\), and we set \(\bm J_{U_l}=\bm 0\) for notational convenience.

Since \(f_{\bm w}(G)\in\mathbb R^K\), the Jacobian matrices of the network output with respect to weight matrix satisfy
\(\operatorname{rank}(\bm J_{U_j})\le K\) and
\(\operatorname{rank}(\bm J_{W_j})\le K\), we write
\[
\bm J_{U_j}
=
\bm Q_{U_j}\bm S_{U_j}\bm V_{U_j}^{T},
\qquad
\bm J_{W_j}
=
\bm Q_{W_j}\bm S_{W_j}\bm V_{W_j}^{T},
\]
\[
\begin{aligned}
\bm S_{U_j}
&=
\diag(s_{U_j,1},\ldots,s_{U_j,K},0,\ldots,0),
\\
\bm S_{W_j}
&=
\diag(s_{W_j,1},\ldots,s_{W_j,K},0,\ldots,0).
\end{aligned}
\]
where \(s_{U_j,1}\ge\cdots\ge s_{U_j,K}\ge0\) and
\(s_{W_j,1}\ge\cdots\ge s_{W_j,K}\ge0\) denote the singular values in the
at most \(K\) output directions.

Accordingly, define the rank-\(K\) sensitivity matrices
\begin{equation}
\label{eq:mpgnn-AU-lr-main}
\bm A_{U_j}
=
\sqrt{2l-1}\,
\bm V_{U_j}
\diag\!\bigl(
\underbrace{Bs,\ldots,Bs}_{K},
0,\ldots,0
\bigr)
\bm V_{U_j}^{T},
\end{equation}
\begin{equation}
\label{eq:mpgnn-AW-lr-main}
\bm A_{W_j}
=
\sqrt{2l-1}\,
\bm V_{W_j}
\diag\!\bigl(
\underbrace{Bs,\ldots,Bs}_{K},
0,\ldots,0
\bigr)
\bm V_{W_j}^{T}.
\end{equation}

Then for any perturbation \(\bm u\), we have
\begin{align}
\|f_{\bm w+\bm u}(G)-f_{\bm w}(G)\|_2^2
&=
\left\|
\sum_{j=1}^{l-1}\bm J_{U_j}\bm u_{U_j}
+
\sum_{j=1}^{l}\bm J_{W_j}\bm u_{W_j}
\right\|_2^2
\notag\\
&\hspace*{-5.2em}\le
\sum_{j=1}^{l-1}\|\bm A_{U_j}\bm u_{U_j}\|_2^2
+
\sum_{j=1}^{l}\|\bm A_{W_j}\bm u_{W_j}\|_2^2.
\label{eq:mpgnn-lr-output-control-main}
\end{align}
\end{lemma}

\begin{proof}
See Appendix~\ref{app:proof-mpgnn-jacobian-spectral}.
\end{proof}

\begin{remark}
Lemma~\ref{lem:mpgnn-lr-sensitivity-domination} constructs the sensitivity
matrices in the standard setting. The key observation is that each
blockwise output Jacobian maps parameter perturbations to the
\(K\)-dimensional graph-level output, and therefore has rank at most
\(K\), regardless of the ambient dimension of the corresponding weight
block. Hence it is sufficient to control the Jacobian action on these
\(K\) active singular directions, rather than uniformly over all parameter
directions.
$\hfill\square$
\end{remark}

The output control in \eqref{eq:mpgnn-lr-output-control-main} verifies the output-stability condition in~\eqref{eq:yi-opt-problem-c}, since \(\|\cdot\|_{\infty}\le\|\cdot\|_2\), and can also be
translated into pairwise margin control, for any \(a,b\in[K]\),
\[
\begin{aligned}
\MoveEqLeft \left|
M(f_{\bm w+\bm u}(G),a,b)
-
M(f_{\bm w}(G),a,b)
\right|^2
\\
&\le
4\|f_{\bm w+\bm u}(G)-f_{\bm w}(G)\|_2^2
\\
&\le
4\left(
\sum_{j=1}^{l-1}\|\bm A_{U_j}\bm u_{U_j}\|_2^2
+
\sum_{j=1}^{l}\|\bm A_{W_j}\bm u_{W_j}\|_2^2
\right).
\end{aligned}
\]
This margin form will be used to pass from standard perturbation control
to robust-margin perturbation control. 

The scale \(s\) gives a common
spectral envelope for the blockwise Jacobians: it captures the geometric
accumulation of message passing when \(\tau\neq1\), and the linear
accumulation in the critical case \(\tau=1\). In the adversarial setting,
the same rank-\(K\) directions are retained, while the active scale is
enlarged from \(B\) to \(B+\epsilon\), and the final bounds preserve
the graph-spectral dependence through \(\tau\) and \(\beta\), while
replacing ambient dimension dependence by the \(K\)-dependent low-rank
term.

\begin{lemma}[MPGNN robust margin perturbation with optimized variance]
\label{lem:mpgnn-perturbation-sigma}
Let \(f_{\bm w}\) be an \(l\)-layer MPGNN with \(l>1\), and let
\(\bm A_{U_j}^{\epsilon}\) and \(\bm A_{W_j}^{\epsilon}\) be the
adversarial sensitivity matrices defined below in
\eqref{eq:mpgnn-adversarial-sensitivity-matrices}. Consider the prior
\(P=\mathcal N(\bm 0,\sigma^2\bm I)\) and the perturbation distribution
\(\bm u\sim\mathcal N(\bm 0,\sigma^2\bm R^{\epsilon})\), where
\(\bm R^{\epsilon}\) is block diagonal with
\[
\begin{aligned}
\bm R_{U_j}^{\epsilon}
&=
(\bm I+\eta^2(\bm A_{U_j}^{\epsilon})^{T}\bm A_{U_j}^{\epsilon})^{-1},
\\
\bm R_{W_j}^{\epsilon}
&=
(\bm I+\eta^2(\bm A_{W_j}^{\epsilon})^{T}\bm A_{W_j}^{\epsilon})^{-1},
\end{aligned}
\]
where \(\eta^2=16\kappa\|\bm w\|_2^2/\gamma^2\), with
\(\kappa=1+2\ln 2+\sqrt{4\ln 2}\).

Since the prior variance cannot depend on the learned value of \(\beta\),
fix a cover point \(\hat\beta\) such that
\[
|\beta-\hat\beta|
\le
\begin{cases}
\dfrac{1}{l+1}\beta,
& \tau\neq 1,\\[0.8em]
\dfrac{1}{3}\beta,
& \tau=1,
\end{cases}
\]
and define
\[
\hat s=
\begin{cases}
\hat\beta^l,
& \tau\neq 1,\\[0.4em]
(l-1)\hat\beta^2,
& \tau=1.
\end{cases}
\]

Consequently, with probability at least \(1/2\) over \(\bm u\),
\begin{equation}
\label{eq:mpgnn-margin-perturbation-main}
\begin{aligned}
\MoveEqLeft\left|
M(f_{\bm w+\bm u}(G),a,b)
-
M(f_{\bm w}(G),a,b)
\right|^2
\\
&\le
4e^2\sigma^2\kappa(2l-1)^2K
B^2\hat s^{\,2}.
\end{aligned}
\end{equation}

Furthermore, under the \(\epsilon\)-attack with \(\epsilon>0\), with
probability at least \(1/2\) over \(\bm u\), the robust margin
perturbation satisfies
\begin{equation}
\label{eq:mpgnn-robust-margin-perturbation-main}
\begin{aligned}
\MoveEqLeft\left|
RM(f_{\bm w+\bm u}(G),a,b)
-
RM(f_{\bm w}(G),a,b)
\right|^2
\\
&\le
4e^2\sigma^2\kappa(2l-1)^2K
(B+\epsilon)^2\hat s^{\,2}.
\end{aligned}
\end{equation}
Comparing \eqref{eq:mpgnn-margin-perturbation-main} with
\eqref{eq:mpgnn-robust-margin-perturbation-main}, the sensitivity scale is enlarged from \(B\) to \(B+\epsilon\). Thus, under the \(\epsilon\)-attack, define
\[
\bm\Sigma_{\epsilon}
=
\diag\!\bigl(
\underbrace{(B+\epsilon)s,\ldots,(B+\epsilon)s}_{K},
0,\ldots,0
\bigr),
\]
and set the adversarial sensitivity matrices as
\begin{equation}
\label{eq:mpgnn-adversarial-sensitivity-matrices}
\begin{aligned}
\bm A_{U_j}^{\epsilon}
&=
\sqrt{2l-1}\,
\bm V_{U_j}
\bm\Sigma_{\epsilon}
\bm V_{U_j}^{T},
\\
\bm A_{W_j}^{\epsilon}
&=
\sqrt{2l-1}\,
\bm V_{W_j}
\bm\Sigma_{\epsilon}
\bm V_{W_j}^{T}.
\end{aligned}
\end{equation}
By setting
\begin{equation}
\label{eq:mpgnn-robust-sigma-choice-main}
\frac{1}{\sigma^2}
=
\frac{
16e^2\kappa(2l-1)^2K(B+\epsilon)^2\hat s^{\,2}
}{
\gamma^2
},
\end{equation}
with probability at least \(1/2\) over \(\bm u\), we have
\begin{equation}
\label{eq:mpgnn-robust-margin-event-main}
\left|
RM(f_{\bm w+\bm u}(G),a,b)
-
RM(f_{\bm w}(G),a,b)
\right|^2
\le
\frac{\gamma^2}{4}.
\end{equation}
\end{lemma}

\begin{proof}
See Appendix~\ref{app:proof-mpgnn-perturbation-sigma}.
\end{proof}

% \begin{remark}
Lemma~\ref{lem:mpgnn-perturbation-sigma} provides the concrete
probabilistic perturbation construction used in the robust PAC-Bayesian
analysis. Compared with the optimization formulation in the preliminaries, the
standard output-stability constraint is replaced by a robust-margin
stability constraint. Accordingly, the sensitivity and covariance blocks
\(\bm A_j\) and \(\bm R_j\) are replaced by their adversarial counterparts
\(\bm A_j^{\epsilon}\) and \(\bm R_j^{\epsilon}\). This yields the modified optimization formulation in the adversarial setting, i.e.,
\begin{subequations}
\label{eq:mpgnn-robust-opt-problem}
\begin{align}
\min_{\sigma^2,\{\bm R_j^{\epsilon}\},\{\bm A_j^{\epsilon}\}}
\quad&
D_{\mathrm{KL}}(\sigma^2,\{\bm R_j^{\epsilon}\}),
\label{eq:mpgnn-robust-opt-problem-a}
\\[-0.2em]
\text{s.t.}\quad&
\mathbb P_{\bm u\sim\mathcal N(\bm 0,\sigma^2\bm R^{\epsilon})}
\!\left[
\sum_{j=1}^{l}\|\bm A_j^{\epsilon}\bm u_j\|_2^2
<
\frac{\gamma^2}{16}
\right]
\ge
\frac{1}{2},
\label{eq:mpgnn-robust-opt-problem-b}
\\[-0.1em]
&\bigl|
RM\!\left(f_{\bm w+\bm u}(G),a,b\right)
-
RM\!\left(f_{\bm w}(G),a,b\right)
\bigr|^2
\nonumber\\[-0.1em]
&\qquad \qquad \qquad\le
4\sum_{j=1}^{l}\|\bm A_j^{\epsilon}\bm u_j\|_2^2.
\label{eq:mpgnn-robust-opt-problem-c}
\end{align}
\end{subequations}
The adversarial sensitivity matrices in
\eqref{eq:mpgnn-adversarial-sensitivity-matrices} are designed to satisfy
the robust-margin domination requirement in
\eqref{eq:mpgnn-robust-opt-problem-c}. Compared with the standard
sensitivity matrices, they preserve the same Jacobian-aligned rank-\(K\)
directions and only enlarge the active singular-value scale from \(Bs\)
to \((B+\epsilon)s\), which matches the transition from
\eqref{eq:mpgnn-margin-perturbation-main} to
\eqref{eq:mpgnn-robust-margin-perturbation-main}. The variance choice in
\eqref{eq:mpgnn-robust-sigma-choice-main} is then used to meet the
probabilistic perturbation requirement in
\eqref{eq:mpgnn-robust-opt-problem-b}. Finally, the cover point \(\hat\beta\) ensures that the prior variance is chosen independently of the learned spectral scale, while
\(\hat s\) records the corresponding Jacobian envelope in the two regimes
\(\tau\neq1\) and \(\tau=1\).
% \end{remark}

Given the above key lemmas, we are ready to derive the adversarially robust generalization bounds for MPGNNs.

\begin{theorem}[Robust generalization bound for MPGNN]
\label{thm:mpgnn-robust-generalization}
For any \(B>0\) and \(l>1\), let
\(f_{\bm w}\in\mathcal H:\mathcal X\times\mathcal G\to\mathbb R^K\)
be an \(l\)-layer MPGNN. Consider the \(\epsilon\)-attack with \(\epsilon>0\). Then, for any \(\delta,\gamma>0\), with
probability at least \(1-\delta\) over the choice of an i.i.d. size-\(m\)
training set \(S\) according to \(\mathcal D\), for any \(\bm w\), we have, given \(\tau=L^3M_2\), that

If \(\tau\neq1\), then
\begin{equation}
\label{eq:mpgnn-main-robust-bound-tau-neq-1}
\begin{aligned}
R_{\mathcal D,0}(f_{\bm w})
&\le
\hat R_{S,\gamma}(f_{\bm w})
\\
&\hspace*{-1.0em}+
\mathcal O\!\left(
\sqrt{
\frac{
l^2K(B+\epsilon)^2\beta^{2l}\|\bm w\|_2^2
+
\ln\frac{m(l+1)}{\delta}
}{
\gamma^2m
}
}
\right);
\end{aligned}
\end{equation}

If \(\tau=1\), then 
\begin{equation}
\label{eq:mpgnn-main-robust-bound-tau-eq-1}
\begin{aligned}
R_{\mathcal D,0}(f_{\bm w})
&\le
\hat R_{S,\gamma}(f_{\bm w})
\\
&\hspace*{-1.0em}+
\mathcal O\!\left(
\sqrt{
\frac{
l^4K(B+\epsilon)^2\beta^4\|\bm w\|_2^2
+
\ln\frac{m}{\delta}
}{
\gamma^2m
}
}
\right),
\end{aligned}
\end{equation}
where
\(\|\bm w\|_2^2
=
\sum_{j=1}^{l-1}\|\bm U_j\|_F^2
+
\sum_{j=1}^{l}\|\bm W_j\|_F^2\), and
\(\beta\) is defined in \eqref{eq:mpgnn-beta-definition}.

\end{theorem}

\begin{proof}
See Appendix~\ref{app:proof-mpgnn-robust-generalization}.
\end{proof}

\begin{remark}
Theorem~\ref{thm:mpgnn-robust-generalization} gives the final robust
PAC-Bayesian bound obtained from the Jacobian-based anisotropic
perturbation analysis. The adversarial perturbation enters the
bound through attacked feature matrix with maximum row
Euclidean norm at most \(B+\epsilon\), when the attacked adjacency matrix can be arbitrary, the corresponding normalized propagation operator remains bounded by the
worst-case with \(\|\bm P'_{G}\|_2 = 1\).
$\hfill\square$
\end{remark}
\begin{remark}
Compared with the existing isotropic robust PAC-Bayesian framework, the
proposed analysis improves both the dimensional and spectral dependence of the bound. Existing bounds control random perturbations in the full
parameter space and therefore contain a hidden-width factor
\(h\log(lh)\), where \(l\) is the number of layers and \(h\) is the maximum hidden width~\cite{Liao-PACBayes-GNN,Sun-PACBayes-Robust}. In contrast, our bound exploits the rank structure of the blockwise output Jacobians, so the leading dimensional factor becomes \(K\), the number of classes. Since typically \(K\ll h\log(lh)\) in graph classification tasks, this yields a tighter dimensional dependence.
The spectral dependence is also refined for the non-homogeneous MPGNN
architecture. Taking \(\tau\neq1\) as an example, the previous robust
MPGNN bound sets
\(\beta=\max\{\zeta^{-1},\xi^{1/l}\}\), where
\(\zeta=\min\bigl(\{\|\bm U_j\|_2\}_{j=1}^{l-1},
\{\|\bm W_j\|_2\}_{j=1}^{l}\bigr)\) denotes the smallest spectral norm
among the \(2l-1\) weight blocks~\cite{Sun-PACBayes-Robust}. In contrast, our bound only uses the scale \(\xi^{1/l}\). Hence the two
scales coincide when \(\zeta^{-1}\le\xi^{1/l}\), while our bound is
strictly tighter when \(\zeta^{-1}\) dominates. In particular, if the
smallest layer spectral norm is only slightly positive, then
\(\zeta^{-1}\) can be very large, making the previous \(\beta\) much larger than \(\xi^{1/l}\). Thus, the proposed bound preserves the
graph-propagation dependence while removing this reciprocal small-layer
factor.
$\hfill\square$
\end{remark}
\begin{remark}
The generalization bounds obtained  in~\cite{Generalization-Limits-GNN,Liao-PACBayes-GNN} for parameter-sharing MPGNNs are special cases of Theorem~\ref{thm:mpgnn-robust-generalization}. Note that such parameter-sharing MPGNNs are obtained by tying the layerwise weights in the present MPGNN formulation,
where \(\bm U_j=\bm W_{1}^{\mathrm{sh}}\) and
\(\bm W_j=\bm W_{2}^{\mathrm{sh}}\) for \(j\in[l-1]\), and the readout matrix is kept separate. The corresponding shared-parameter Jacobians are \(\bm J_{\bm W_{1}}^{\mathrm{sh}}=\sum_{j=1}^{l-1}\bm J_{U_j}\) and
\(\bm J_{\bm W_{2}}^{\mathrm{sh}}=\sum_{j=1}^{l-1}\bm J_{W_j}\). Since the layerwise Jacobians are dominated by the common sensitivity scale
\(Bs\), then \(\|\bm J_{\bm W_{2}}^{\mathrm{sh}}\|_2\le(l-1)Bs\), and the same bound holds for \(\bm J_{\bm W_{1}}^{\mathrm{sh}}\) and \(\bm J_{W_l}\). Thus, parameter sharing
introduces at most an \(O(l)\) factor in the shared Jacobian norm, but
reduces the number of independent perturbation blocks from \(O(l)\) to
\(O(1)\). Hence the margin-based perturbation bound remains of order
\(O(l^2K(B+\epsilon)^2s^2)\), leading to the same robust generalization
bound as in Theorem~\ref{thm:mpgnn-robust-generalization}.
$\hfill\square$
\end{remark}

\begin{corollary}[Robust generalization bound for GCNs]
\label{cor:gcn-robust-generalization}
Under the same mild assumptions, consider the \(\epsilon\)-attack with
\(\epsilon>0\). Let \(f_{\bm w}\) be an \(l\)-layer GCN with ReLU
activations and normalized propagation operator \(\|\bm P_G\|_2\le1\).
Then, for any \(\delta,\gamma>0\), with probability at least \(1-\delta\)
over the choice of an i.i.d. size-\(m\) training set
\(S\sim\mathcal D^m\), for any \(\bm w\), we have
\begin{equation}
\label{eq:gcn-main-robust-bound}
\begin{aligned}
R_{\mathcal D,0}(f_{\bm w})
&\le
\hat R_{S,\gamma}(f_{\bm w})
\\
&\hspace*{-1.0em}+
\mathcal O\!\left(
\sqrt{
\frac{
l^2K(B+\epsilon)^2\beta^{2l-2}\|\bm w\|_2^2
+
\ln\frac{ml}{\delta}
}{
\gamma^2m
}
}
\right)
\end{aligned}
\end{equation}
where
\(\beta=\left(\prod_{j=1}^{l}\|\bm W_j\|_2\right)^{1/l}\), and
\(\|\bm w\|_2^2=\sum_{j=1}^{l}\|\bm W_j\|_F^2\) after the homogeneous
normalization \(\|\bm W_j\|_2=\beta\).
\end{corollary}

\begin{proof}
See Appendix~\ref{app:proof-gcn-robust-generalization}.
\end{proof}

\begin{remark}
Corollary~\ref{cor:gcn-robust-generalization} shows that the proposed
MPGNN analysis naturally recovers the GCN case~\cite{Liao-PACBayes-GNN,Sun-PACBayes-Robust}. The key difference is that GCNs are homogeneous in the layer weights, so the weights can be
spectrally normalized without changing the represented function. With
\(\beta=\bigl(\prod_{j=1}^{l}\|\bm W_j\|_2\bigr)^{1/l}\) and
\(\tilde{\bm W}_j=\frac{\beta}{\|\bm W_j\|_2}\bm W_j\), we have
\(f_{\bm w}=f_{\tilde{\bm w}}\), and the spectral term becomes
\[
\beta^{2l-2}\|\tilde{\bm w}\|_2^2
=
\prod_{j=1}^{l}\|\bm W_j\|_2^2
\sum_{j=1}^{l}
\frac{\|\bm W_j\|_F^2}{\|\bm W_j\|_2^2}.
\]
Thus, the bound recovers the classical product-norm spectral complexity
for homogeneous ReLU architectures. Unlike general MPGNNs, whose
non-homogeneous input-injection and message-passing terms require a common spectral envelope, GCNs admit an exact layerwise spectral normalization. At the same time, the low-rank Jacobian construction replaces the hidden-width factor \(h\log(lh)\) in isotropic PAC-Bayesian bounds by the class-dependent factor \(K\), preserving the improved dimension dependence.
$\hfill\square$
\end{remark}

\section{Conclusion}
\label{sec:conclusion}

This paper derived a tighter PAC-Bayesian generalization bound for MPGNNs in the adversarial setting using a Jacobian-based sensitivity analytical approach. The analysis shows that, for graph classification, the output sensitivity of MPGNNs admits a
low-rank structure induced by the \(K\)-dimensional graph-level output. By
exploiting this structure together with optimized anisotropic Gaussian
posteriors, we obtain robust generalization bounds with sharper spectral
dependence and improved dimensional dependence compared with existing
state-of-the-art methods. In particular, the spectral scale associated with the
learned weights is tightened, and the leading dimension factor is reduced
from a hidden-width-dependent term to the number of classes. The framework
also recovers GCN-type models as special cases, yielding a unified robust
generalization analysis for representative GNN architectures.

Several directions remain for future work. Although the present analysis
allows arbitrary topology perturbations, their effect is constrained by the
worst-case bound \(\|\bm P'_{G}\|_2\le 1\), which does not distinguish different edge perturbation patterns or describe how graph changes affect the propagation operator. A finer analysis of the impact of topology change
\(\bm P'_{G}-\bm P_G\) may yield robust bounds that more explicitly capture the role of topology attacks.
Another direction is to refine the sensitivity matrix design as in \cite{yi2026topology} to account for graph structures. The present
Jacobian-aligned low-rank construction exploits the \(K\)-dimensional
output structure, but does not fully leverage the data-dependent graph geometry
or architecture-specific constraints. Incorporating graph spectra, learned
representations, or a message-passing structure, may further tighten the
complexity term and improve the interpretability of robust generalization
bounds for graph neural networks.
\appendix
\section{Proofs of Main Results}
\label{app:proofs}

\subsection{Proof of Lemma~\ref{lem:mpgnn-Jacobian-UW}}
\label{app:proof-mpgnn-Jacobian-UW}
\begin{proof}
For the readout layer, we have
\begin{equation}
\label{eq:app-readout-vec}
\mathrm{vec}\!\left(f_{\bm w}(G)\right)
=
\frac{1}{n}
\left(
\bm W_l^{T}\otimes \bm 1_n^{T}
\right)
\mathrm{vec}(\bm H_{l-1}),
\end{equation}
which yields
\begin{equation}
\label{eq:app-readout-jacobian}
\frac{\partial f_{\bm w}(G)}
{\partial \mathrm{vec}(\bm H_{l-1})}
=
\frac{1}{n}
\left(
\bm W_l^{T}\otimes \bm 1_n^{T}
\right).
\end{equation}
For \(k=j+1,\ldots,l-1\), by the chain rule, we have
\begin{equation}
\label{eq:app-hidden-step-jacobian}
\begin{aligned}
\frac{\partial \mathrm{vec}(\bm H_k)}
{\partial \mathrm{vec}(\bm H_{k-1})}
&=
\bm B_k^{\phi}
\frac{\partial \mathrm{vec}(\bm E_k)}
{\partial \mathrm{vec}(\bm H_{k-1})}
\\
&=
\bm B_k^{\phi}\bm B_k^{\rho}
\frac{\partial \mathrm{vec}(\bm F_k)}
{\partial \mathrm{vec}(\bm H_{k-1})}
\\
&=
\bm B_k^{\phi}\bm B_k^{\rho}
\left(
\bm W_k^{T}\otimes \bm P_G
\right)
\bm B_k^{\psi}.
\end{aligned}
\end{equation}
Therefore,
\begin{equation}
\label{eq:app-Mj-proof}
\begin{aligned}
\bm M_j
&=
\frac{\partial f_{\bm w}(G)}
{\partial \mathrm{vec}(\bm H_{l-1})}
\prod_{k=l-1}^{j+1}
\frac{\partial \mathrm{vec}(\bm H_k)}
{\partial \mathrm{vec}(\bm H_{k-1})}
\\
&=
\frac{1}{n}
\left(
\bm W_l^{T}\otimes \bm 1_n^{T}
\right)
\left(
\prod_{k=l-1}^{j+1}
\bm B_k^{\phi}\bm B_k^{\rho}
\left(
\bm W_k^{T}\otimes \bm P_G
\right)
\bm B_k^{\psi}
\right),
\end{aligned}
\end{equation}
which proves \eqref{eq:mpgnn-Mj-main}.

Next, for the input-injection weight matrix \(\bm U_j\), the perturbation enters
the \(j\)-th layer through \(\bm X\bm U_j\),
\begin{equation}
\label{eq:app-Uj-local}
\frac{\partial \mathrm{vec}(\bm H_j)}
{\partial \mathrm{vec}(\bm U_j)}
=
\bm B_j^{\phi}
\frac{\partial \mathrm{vec}(\bm X\bm U_j)}
{\partial \mathrm{vec}(\bm U_j)}
=
\bm B_j^{\phi}
\left(
\bm I_{h_j}\otimes \bm X
\right).
\end{equation}
Using the chain rule gives
\begin{equation}
\label{eq:app-JUj-proof}
\bm J_{U_j}
=
\frac{\partial f_{\bm w}(G)}
{\partial \mathrm{vec}(\bm H_j)}
\frac{\partial \mathrm{vec}(\bm H_j)}
{\partial \mathrm{vec}(\bm U_j)}
=
\bm M_j\bm B_j^{\phi}
\left(
\bm I_{h_j}\otimes \bm X
\right),
\end{equation}
which proves \eqref{eq:mpgnn-JUj-main}.

Similarly, for the message-passing weight matrix \(\bm W_j\), the perturbation
enters the \(j\)-th layer through
\(\bm P_G\psi_j(\bm H_{j-1})\bm W_j\), such that
\begin{equation}
\label{eq:app-Wj-local}
\begin{aligned}
\frac{\partial \mathrm{vec}(\bm H_j)}
{\partial \mathrm{vec}(\bm W_j)}
&=
\bm B_j^{\phi}\bm B_j^{\rho}
\frac{\partial \mathrm{vec}\!\left(
\bm P_G\psi_j(\bm H_{j-1})\bm W_j
\right)}
{\partial \mathrm{vec}(\bm W_j)}
\\
&=
\bm B_j^{\phi}\bm B_j^{\rho}
\left(
\bm I_{h_j}\otimes
\bm P_G\psi_j(\bm H_{j-1})
\right),
\end{aligned}
\end{equation}
\begin{equation}
\label{eq:app-JWj-proof}
\begin{aligned}
\bm J_{W_j}
&=
\bm M_j\bm B_j^{\phi}\bm B_j^{\rho}
\left(
\bm I_{h_j}\otimes
\bm P_G\psi_j(\bm H_{j-1})
\right),
\end{aligned}
\end{equation}
which proves \eqref{eq:mpgnn-JWj-main}. In particular, for \(j=1\),
since \(\psi_1(\bm H_0)=\bm 0\), we obtain \(\bm J_{W_1}=\bm 0\).
\end{proof}

\subsection{Proof of Lemma~\ref{lem:mpgnn-lr-sensitivity-domination}}
\label{app:proof-mpgnn-jacobian-spectral}
\begin{proof}
From Lemma~\ref{lem:mpgnn-Jacobian-UW}, we first bound \(\|\bm M_j\|_2\) as
\begin{equation}
\label{eq:app-Mj-spectral-proof}
\begin{aligned}
\|\bm M_j\|_2
&\le
\frac{\|\bm W_l\|_2}{\sqrt n}
\prod_{k=j+1}^{l-1}
\left\|
\bm B_k^{\phi}\bm B_k^{\rho}
(\bm W_k^{T}\otimes \bm P_G)\bm B_k^{\psi}
\right\|_2
\\
&\le
\frac{\|\bm W_l\|_2}{\sqrt n}
\prod_{k=j+1}^{l-1}
L^3\|\bm W_k\|_2\|\bm P_G\|_2
\\
&\le
\frac{M_2}{\sqrt n}\,
\tau^{\,l-1-j}.
\end{aligned}
\end{equation}

For the input-injection weight matrix \(\bm U_j\), we have
\begin{equation}
\label{eq:app-JUj-spectral-proof}
\begin{aligned}
\|\bm J_{U_j}\|_2
&\le
\|\bm M_j\|_2
\|\bm B_j^{\phi}\|_2
\left\|
\bm I_{h_j}\otimes \bm X
\right\|_2
\\
&\le
\frac{L M_2\|\bm X\|_2}{\sqrt n}\,
\tau^{\,l-1-j}
\\
&\le
B L M_2\,\tau^{\,l-1-j},
\end{aligned}
\end{equation}
where the last inequality uses
\begin{equation}
\label{eq:X-spectral-to-2inf}
\|\bm X\|_2
\le
\|\bm X\|_F
\le
\sqrt n\,\|\bm X\|_{2,\infty}
\le
\sqrt n\,B.
\end{equation}

For the message-passing weight matrix \(\bm W_j\), we have
\begin{equation}
\label{eq:app-JWj-spectral-proof}
\begin{aligned}
\|\bm J_{W_j}\|_2
&\le
\|\bm M_j\|_2
\|\bm B_j^{\phi}\|_2
\|\bm B_j^{\rho}\|_2
\|\bm P_G\psi_j(\bm H_{j-1})\|_2
\\
&\le
L^2\|\bm M_j\|_2
\|\bm P_G\|_2
\|\psi_j(\bm H_{j-1})\|_2
\\
&\le
L^3\|\bm M_j\|_2
\|\bm P_G\|_2
\|\bm H_{j-1}\|_F
\\
&\le
\frac{1}{\sqrt n}
\tau^{\,l-j}
\|\bm H_{j-1}\|_F,
\end{aligned}
\end{equation}
where \(\psi_j\) is \(L\)-Lipschitz and \(\psi_j(\bm 0)=\bm 0\), such that
\begin{equation}
\label{eq:psi-H-fro-bound}
\|\psi_j(\bm H_{j-1})\|_2
\le
\|\psi_j(\bm H_{j-1})\|_F
\le
L\|\bm H_{j-1}\|_F .
\end{equation}
It remains to bound \(\|\bm H_{j-1}\|_F\), which is given by
\begin{equation}
\label{eq:app-Hj-F-recursion}
\begin{aligned}
\|\bm H_j\|_F
&=
\|\phi_j(\bm E_j)\|_F
\\
&\le
L\|\bm X\bm U_j+\rho_j(\bm F_j)\|_F
\\
&\le
L\|\bm X\bm U_j\|_F
+
L\|\rho_j(\bm F_j)\|_F
\\
&\le
L M_1\|\bm X\|_F
+
L^2\|\bm F_j\|_F
\\
&\le
L M_1\|\bm X\|_F
+
L^2\|\bm P_G\|_2
\|\psi_j(\bm H_{j-1})\|_F
\|\bm W_j\|_2
\\
&\le
L M_1\|\bm X\|_F
+
L^3M_2\|\bm P_G\|_2
\|\bm H_{j-1}\|_F
\\
&=
L M_1\|\bm X\|_F
+
\tau\|\bm H_{j-1}\|_F .
\end{aligned}
\end{equation}
Since \(\bm H_0=\bm 0\), iterating \eqref{eq:app-Hj-F-recursion} gives
\begin{equation}
\label{eq:app-Hj-F-bound}
\|\bm H_j\|_F
\le
L M_1\|\bm X\|_F
\sum_{k=0}^{j-1}\tau^k ,
\end{equation}
\begin{equation}
\label{eq:app-Hjminus-F-bound}
\|\bm H_{j-1}\|_F
\le
L M_1\|\bm X\|_F
\sum_{k=0}^{j-2}\tau^k
\le
\sqrt n\,B L M_1
\sum_{k=0}^{j-2}\tau^k .
\end{equation}
Combining \eqref{eq:app-JWj-spectral-proof} and
\eqref{eq:app-Hjminus-F-bound}, we obtain
\begin{equation}
\label{eq:app-JWj-spectral-final}
\|\bm J_{W_j}\|_2
\le
B L M_1
\tau^{\,l-j}
\sum_{k=0}^{j-2}\tau^k .
\end{equation}
Define \(\tau=L^3M_2\|\bm P_G\|_2\), with the bound \(\|\bm P_G\|_2\le 1\), and
\[
\beta=
\begin{cases}
\left(
LM_1M_2\dfrac{\tau^{l-1}-1}{\tau-1}
\right)^{1/l},
& \tau\neq 1,\\[0.8em]
\sqrt{LM_1M_2},
& \tau=1.
\end{cases}
\]
By \eqref{eq:app-JUj-spectral-proof} and \eqref{eq:app-JWj-spectral-final}, 
with \(\bm J_{U_l}=\bm 0\), \(M_1,M_2,L,\tau>0\), then
\begin{equation}
\label{eq:mpgnn-common-jacobian-bound}
\max_{j\in[l]}
\left\{
\|\bm J_{U_j}\|_2,\,
\|\bm J_{W_j}\|_2
\right\}
\le
\begin{cases}
B\beta^l,
& \tau\neq 1,\\[0.4em]
B(l-1)\beta^2,
& \tau=1.
\end{cases}
\end{equation}
Since \(f_{\bm w}(G)\in\mathbb R^K\), the Jacobian matrices of the network output with respect to weight matrix satisfy
\(\operatorname{rank}(\bm J_{U_j})\le K\) and
\(\operatorname{rank}(\bm J_{W_j})\le K\), hence we have
\[
\bm J_{U_j}
=
\bm Q_{U_j}\bm S_{U_j}\bm V_{U_j}^{T},
\qquad
\bm J_{W_j}
=
\bm Q_{W_j}\bm S_{W_j}\bm V_{W_j}^{T},
\]
\[
\bm S_{U_j}
=
\diag(s_{U_j,1},\ldots,s_{U_j,K},0,\ldots,0),
\]
\[
\bm S_{W_j}
=
\diag(s_{W_j,1},\ldots,s_{W_j,K},0,\ldots,0),
\]
where \(s_{U_j,1}\ge\cdots\ge s_{U_j,K}\ge0\) and
\(s_{W_j,1}\ge\cdots\ge s_{W_j,K}\ge0\) denote the singular values in the
at most \(K\) output directions.

Accordingly, define the rank-\(K\) sensitivity matrices
\begin{equation}
\label{eq:mpgnn-AU-lr}
\bm A_{U_j}
=
\sqrt{2l-1}\,
\bm V_{U_j}
\diag\!\bigl(
\underbrace{Bs,\ldots,Bs}_{K},
0,\ldots,0
\bigr)
\bm V_{U_j}^{T},
\end{equation}
\begin{equation}
\label{eq:mpgnn-AW-lr}
\bm A_{W_j}
=
\sqrt{2l-1}\,
\bm V_{W_j}
\diag\!\bigl(
\underbrace{Bs,\ldots,Bs}_{K},
0,\ldots,0
\bigr)
\bm V_{W_j}^{T},
\end{equation}
where
\[
s=
\begin{cases}
\beta^l,
& \tau\neq 1,\\[0.4em]
(l-1)\beta^2,
& \tau=1.
\end{cases}
\]
It is clear that
\[
\bm J_{U_j}^{T}\bm J_{U_j}
\preceq
\frac{1}{2l-1}\bm A_{U_j}^{T}\bm A_{U_j},
\quad
\bm J_{W_j}^{T}\bm J_{W_j}
\preceq
\frac{1}{2l-1}\bm A_{W_j}^{T}\bm A_{W_j}.
\]
Therefore, for any perturbation \(\bm u\), we have
\begin{align}
\|f_{\bm w+\bm u}(G)-f_{\bm w}(G)\|_2^2
&=
\left\|
\sum_{j=1}^{l-1}\bm J_{U_j}\bm u_{U_j}
+
\sum_{j=1}^{l}\bm J_{W_j}\bm u_{W_j}
\right\|_2^2
\notag\\
&\hspace*{-6.6em}\le
(2l-1)
\left(
\sum_{j=1}^{l-1}\|\bm J_{U_j}\bm u_{U_j}\|_2^2
+
\sum_{j=1}^{l}\|\bm J_{W_j}\bm u_{W_j}\|_2^2
\right)
\notag\\
&\hspace*{-6.6em}\le
\sum_{j=1}^{l-1}\|\bm A_{U_j}\bm u_{U_j}\|_2^2
+
\sum_{j=1}^{l}\|\bm A_{W_j}\bm u_{W_j}\|_2^2 .
\label{eq:mpgnn-lr-output-control}
\end{align}
Since \(\|\cdot\|_{\infty}\le\|\cdot\|_2\), 
\eqref{eq:mpgnn-lr-output-control} directly implies the
output-stability condition in \eqref{eq:yi-opt-problem-c}. It also yields
a corresponding pairwise margin control, for any \(a,b\in[K]\),
by the definition of \(M\),
\begin{equation}
\label{eq:mpgnn-margin-output-diff-proof}
\begin{aligned}
\MoveEqLeft \left|
M(f_{\bm w+\bm u}(G),a,b)
-
M(f_{\bm w}(G),a,b)
\right|
\\
&=
\left|
\bigl(f_{\bm w+\bm u}(G)_a-f_{\bm w+\bm u}(G)_b\bigr)
-
\bigl(f_{\bm w}(G)_a-f_{\bm w}(G)_b\bigr)
\right|
\\
&\le
\left|
f_{\bm w+\bm u}(G)_a-f_{\bm w}(G)_a
\right|
+
\left|
f_{\bm w+\bm u}(G)_b-f_{\bm w}(G)_b
\right|
\\
&\le
2\|f_{\bm w+\bm u}(G)-f_{\bm w}(G)\|_2 .
\end{aligned}
\end{equation}
Consequently,
\begin{equation}
\label{eq:mpgnn-margin-control-proof}
\begin{aligned}
\MoveEqLeft \left|
M(f_{\bm w+\bm u}(G),a,b)
-
M(f_{\bm w}(G),a,b)
\right|^2
\\
&\le
4\|f_{\bm w+\bm u}(G)-f_{\bm w}(G)\|_2^2
\\
&\le
4\left(
\sum_{j=1}^{l-1}\|\bm A_{U_j}\bm u_{U_j}\|_2^2
+
\sum_{j=1}^{l}\|\bm A_{W_j}\bm u_{W_j}\|_2^2
\right).
\end{aligned}
\end{equation}
Thus, the same low-rank quadratic control can also be used for the pairwise margin perturbation condition.
This completes the proof.
\end{proof}

\subsection{Proof of Lemma~\ref{lem:mpgnn-perturbation-sigma}}
\label{app:proof-mpgnn-perturbation-sigma}
\begin{proof}
Consider the prior \(P=\mathcal N(\bm 0,\sigma^2\bm I)\) and the random
perturbation \(\bm u\sim\mathcal N(\bm 0,\sigma^2\bm R)\). Note that the
\(\sigma\) of the prior and the perturbation distributions are the same and
will be set according to \(\beta\). More precisely, we set \(\sigma\)
based on cover point \(\hat\beta\) of \(\beta\) since the prior
\(P\) cannot depend on any learned weights directly. 
We fix any \(\hat\beta\) and consider \(\beta\) which satisfies
\[
|\beta-\hat\beta|
\le
\begin{cases}
\dfrac{1}{l+1}\beta,
& \tau\neq 1,\\[0.8em]
\dfrac{1}{3}\beta,
& \tau=1.
\end{cases}
\]
If \(\tau\neq1\), then
\(\frac{1}{e}\beta^l
\le
\hat\beta^l
\le
e\beta^l\);
If \(\tau=1\), then
\(\frac{4}{9}\beta^2
\le
\hat\beta^2
\le
\frac{16}{9}\beta^2\), which also implies
\(\frac{1}{e}\beta^2\le\hat\beta^2\le e\beta^2\).

We introduce the approximations of \(\bm A_{U_j}\) and \(\bm A_{W_j}\), i.e.,
\begin{equation}
\label{eq:mpgnn-AU-hat-lr}
\hat{\bm A}_{U_j}
=
\sqrt{2l-1}\,
\bm V_{U_j}
\diag\!\bigl(
\underbrace{B\hat s,\ldots,B\hat s}_{K},
0,\ldots,0
\bigr)
\bm V_{U_j}^{T},
\end{equation}
\begin{equation}
\label{eq:mpgnn-AW-hat-lr}
\hat{\bm A}_{W_j}
=
\sqrt{2l-1}\,
\bm V_{W_j}
\diag\!\bigl(
\underbrace{B\hat s,\ldots,B\hat s}_{K},
0,\ldots,0
\bigr)
\bm V_{W_j}^{T},
\end{equation}
where
\[
\hat s=
\begin{cases}
\hat\beta^l,
& \tau\neq 1,\\[0.4em]
(l-1)\hat\beta^2,
& \tau=1,
\end{cases}
\]
with \(s\le e\hat s\), we have
\begin{equation}
\label{eq:mpgnn-A-hat-dominate}
\bm A_{U_j}\bm A_{U_j}^{T}
\preceq
e^2\hat{\bm A}_{U_j}\hat{\bm A}_{U_j}^{T},
\quad
\bm A_{W_j}\bm A_{W_j}^{T}
\preceq
e^2\hat{\bm A}_{W_j}\hat{\bm A}_{W_j}^{T}.
\end{equation}
Then optimize the posterior covariance to minimize the KL divergence as in \eqref{eq:yi-R-opt},
such that
\[
\bm R_{U_j}=(\bm I+\eta^2\bm A_{U_j}^{T}\bm A_{U_j})^{-1},
\quad
\bm R_{W_j}=(\bm I+\eta^2\bm A_{W_j}^{T}\bm A_{W_j})^{-1}
\]
where \(\eta^2=16\kappa\|\bm w\|_2^2/\gamma^2\), with \(\kappa=1+2\ln 2+\sqrt{4\ln 2}\), since
\(\bm R_{U_j}\preceq\bm I\) and \(\bm R_{W_j}\preceq\bm I\). By the
Gaussian concentration bound, with probability at least \(1/2\) over
\(\bm u\), we have
\begin{align}
&\sum_{j=1}^{l-1}\|\bm A_{U_j}\bm u_{U_j}\|_2^2
+
\sum_{j=1}^{l}\|\bm A_{W_j}\bm u_{W_j}\|_2^2
\notag\\
&\le
\sigma^2\kappa
\left(
\sum_{j=1}^{l-1}
\Tr(\bm A_{U_j}\bm R_{U_j}\bm A_{U_j}^{T})
+
\sum_{j=1}^{l}
\Tr(\bm A_{W_j}\bm R_{W_j}\bm A_{W_j}^{T})
\right)
\notag\\
&\le
e^2\sigma^2\kappa
\left(
\sum_{j=1}^{l-1}
\Tr(\hat{\bm A}_{U_j}\hat{\bm A}_{U_j}^{T})
+
\sum_{j=1}^{l}
\Tr(\hat{\bm A}_{W_j}\hat{\bm A}_{W_j}^{T})
\right)
\notag\\
&=
e^2\sigma^2\kappa(2l-1)^2K B^2\hat s^{\,2}.
\label{eq:mpgnn-sigma-fixed-beta-condition}
\end{align}
Combining \eqref{eq:mpgnn-sigma-fixed-beta-condition} with
\eqref{eq:mpgnn-lr-output-control}, we obtain
\[
\|f_{\bm w+\bm u}(G)-f_{\bm w}(G)\|_2^2
\le
e^2\sigma^2\kappa(2l-1)^2K B^2\hat s^{\,2}.
\]
By \eqref{eq:mpgnn-margin-output-diff-proof}, for any \(a,b\in[K]\), we have
\begin{equation}
\label{eq:mpgnn-margin-perturbation-proof}
\begin{aligned}
\MoveEqLeft\left|
M(f_{\bm w+\bm u}(G),a,b)
-
M(f_{\bm w}(G),a,b)
\right|^2
\\
&\le
4e^2\sigma^2\kappa(2l-1)^2K B^2\hat s^{\,2}.
\end{aligned}
\end{equation}
Equivalently, letting
\(C_{\sigma}=2e\sigma\sqrt{\kappa K}\,(2l-1)\hat s\), we have
\begin{equation}
\label{eq:mpgnn-margin-input-dependent-linear}
\left|
M(f_{\bm w+\bm u}(G),a,b)
-
M(f_{\bm w}(G),a,b)
\right|
\le
C_{\sigma}B .
\end{equation}

Consider the \(\epsilon\)-attack with \(\epsilon>0\). We choose the adversarial sample
\(G_{\bm w}^{*}
=
\arg\inf_{G'\in\delta_{\bm w}(G)}
M\!\left(f_{\bm w}(G'),a,b\right)\)
with \(\bm X_{\bm w}^{*}\) being the corresponding attacked node feature matrix,
which minimizes the pairwise margin operator. Then we have
\begin{equation}
\label{eq:mpgnn-rm-input-dependent-proof}
\begin{aligned}
\MoveEqLeft\left|
RM(f_{\bm w+\bm u}(G),a,b)
-
RM(f_{\bm w}(G),a,b)
\right|
\\
&=
\left|
M(f_{\bm w+\bm u}(G_{\bm w+\bm u}^{*}),a,b)
-
M(f_{\bm w}(G_{\bm w}^{*}),a,b)
\right|
\\
&\le
\max\Bigl\{
\left|
M(f_{\bm w+\bm u}(G_{\bm w+\bm u}^{*}),a,b)
-
M(f_{\bm w}(G_{\bm w+\bm u}^{*}),a,b)
\right|,
\\
&\hspace*{5.2em}
\left|
M(f_{\bm w+\bm u}(G_{\bm w}^{*}),a,b)
-
M(f_{\bm w}(G_{\bm w}^{*}),a,b)
\right|
\Bigr\}
\\
&\le
C_{\sigma}
\max\left\{
\|\bm X_{\bm w+\bm u}^{*}\|_{2,\infty},
\|\bm X_{\bm w}^{*}\|_{2,\infty}
\right\}
\\
&\le
C_{\sigma}(B+\epsilon).
\end{aligned}
\end{equation}
Therefore, we have
\begin{equation}
\label{eq:mpgnn-robust-margin-perturbation-final}
\begin{aligned}
\MoveEqLeft\left|
RM(f_{\bm w+\bm u}(G),a,b)
-
RM(f_{\bm w}(G),a,b)
\right|^2
\\
&\le
4e^2\sigma^2\kappa(2l-1)^2K
(B+\epsilon)^2\hat s^{\,2}.
\end{aligned}
\end{equation}

We now verify the two perturbation constraints in
\eqref{eq:mpgnn-robust-opt-problem-b} and
\eqref{eq:mpgnn-robust-opt-problem-c}. Comparing
\eqref{eq:mpgnn-margin-perturbation-proof} with
\eqref{eq:mpgnn-robust-margin-perturbation-final}, the robust margin
perturbation has the same low-rank sensitivity structure as the standard
margin perturbation, while the active scale is enlarged from \(B\) to
\(B+\epsilon\). Therefore, to dominate the robust margin perturbation, it
is sufficient to keep the Jacobian-aligned rank-\(K\) directions in
\eqref{eq:mpgnn-AU-lr} and \eqref{eq:mpgnn-AW-lr}, and replace the
nonzero singular values \(Bs\) by \((B+\epsilon)s\). Define
\[
\bm\Sigma_{\epsilon}
=
\diag\!\bigl(
\underbrace{(B+\epsilon)s,\ldots,(B+\epsilon)s}_{K},
0,\ldots,0
\bigr).
\]
Then the adversarial sensitivity matrices are
\[
\begin{aligned}
\bm A_{U_j}^{\epsilon}
&=
\sqrt{2l-1}\,
\bm V_{U_j}\bm\Sigma_{\epsilon}\bm V_{U_j}^{T},
\\
\bm A_{W_j}^{\epsilon}
&=
\sqrt{2l-1}\,
\bm V_{W_j}\bm\Sigma_{\epsilon}\bm V_{W_j}^{T}.
\end{aligned}
\]
With these choices, the blockwise form of
\eqref{eq:mpgnn-robust-opt-problem-c} becomes
\begin{equation}
\label{eq:mpgnn-robust-margin-domination-proof}
\begin{aligned}
\MoveEqLeft\left|
RM(f_{\bm w+\bm u}(G),a,b)
-
RM(f_{\bm w}(G),a,b)
\right|^2
\\
&\le
4\left(
\sum_{j=1}^{l-1}
\|\bm A_{U_j}^{\epsilon}\bm u_{U_j}\|_2^2
+
\sum_{j=1}^{l}
\|\bm A_{W_j}^{\epsilon}\bm u_{W_j}\|_2^2
\right)
\\
&\le
4e^2\sigma^2\kappa(2l-1)^2K
(B+\epsilon)^2\hat s^{\,2},
\end{aligned}
\end{equation}
where the last inequality follows from
\eqref{eq:mpgnn-sigma-fixed-beta-condition}.

Next, for fixed adversarial sensitivity matrices, the covariance blocks
are chosen by minimizing the KL term with respect to \(\bm R\). Applying
\eqref{eq:yi-R-opt} with \(\bm A_j\) replaced by
\(\bm A_j^{\epsilon}\), we obtain
\[
\begin{aligned}
\bm R_{U_j}^{\epsilon,*}
&=
\left(
\bm I
+
\eta^2
(\bm A_{U_j}^{\epsilon})^{T}\bm A_{U_j}^{\epsilon}
\right)^{-1},
\\
\bm R_{W_j}^{\epsilon,*}
&=
\left(
\bm I
+
\eta^2
(\bm A_{W_j}^{\epsilon})^{T}\bm A_{W_j}^{\epsilon}
\right)^{-1}.
\end{aligned}
\]

By \eqref{eq:mpgnn-robust-margin-domination-proof}, it is sufficient to
choose \(\sigma^2\) so that
\[
4e^2\sigma^2\kappa(2l-1)^2K(B+\epsilon)^2\hat s^{\,2}
\le
\frac{\gamma^2}{4}.
\]
Thus, setting
\begin{equation}
\label{eq:mpgnn-robust-sigma-choice-final}
\frac{1}{\sigma^2}
=
\frac{
16e^2\kappa(2l-1)^2K(B+\epsilon)^2\hat s^{\,2}
}{
\gamma^2
},
\end{equation}
gives, with probability at least \(1/2\) over \(\bm u\), that
\[
\left|
RM(f_{\bm w+\bm u}(G),a,b)
-
RM(f_{\bm w}(G),a,b)
\right|^2
\le
\frac{\gamma^2}{4}.
\]
This verifies the perturbation requirement in
\eqref{eq:mpgnn-robust-opt-problem-b} together with
\eqref{eq:mpgnn-robust-opt-problem-c} and completes the proof.
\end{proof}

\subsection{Proof of Theorem~\ref{thm:mpgnn-robust-generalization}}
\label{app:proof-mpgnn-robust-generalization}
\begin{proof}
In the adversarial setting, the posterior covariance is chosen according to the adversarial sensitivity matrices with the optimized adversarial covariance blocks, i.e.,
\begin{align}
\bm R_{U_j}^{\epsilon,*}
&=
\left(
\bm I
+
\eta^2(2l-1)
\bm V_{U_j}\bm\Sigma_{\epsilon}^{2}\bm V_{U_j}^{T}
\right)^{-1}
\notag\\
&=
\bm V_{U_j}
\left(
\bm I+\eta^2(2l-1)\bm\Sigma_{\epsilon}^{2}
\right)^{-1}
\bm V_{U_j}^{T},
\label{eq:mpgnn-optimal-RUj-epsilon}
\end{align}
\begin{equation}
\label{eq:mpgnn-optimal-RWj-epsilon}
\bm R_{W_j}^{\epsilon,*}
=
\bm V_{W_j}
\left(
\bm I+\eta^2(2l-1)\bm\Sigma_{\epsilon}^{2}
\right)^{-1}
\bm V_{W_j}^{T},
\end{equation}
where \(\eta^2=16\kappa\|\bm w\|_2^2/\gamma^2\). The optimized
covariance has the same form as in the standard sensitivity framework,
but its shrinkage is determined by the adversarial sensitivity scale
\((B+\epsilon)s\).
Plugging the choice of \(1/\sigma^2\) in Lemma~\ref{lem:mpgnn-perturbation-sigma} into the KL
objective \eqref{eq:yi-opt-problem-a}, we have
\begin{align}
D_{\mathrm{KL}}
&\le
\frac{1}{2}
\Biggl[
\frac{\|\bm w\|_2^2}{\sigma^2}
+
\sum_{j=1}^{l-1}
\Bigl(
\Tr(\bm R_{U_j}^{\epsilon,*})
-
\log\det(\bm R_{U_j}^{\epsilon,*})
-
h^2
\Bigr)
\notag\\
&\hspace*{3.2em}
+
\sum_{j=1}^{l}
\Bigl(
\Tr(\bm R_{W_j}^{\epsilon,*})
-
\log\det(\bm R_{W_j}^{\epsilon,*})
-
h^2
\Bigr)
\Biggr]
\notag\\
&\le
\frac{8e^2\kappa\|\bm w\|_2^2}{\gamma^2}
(2l-1)^2K(B+\epsilon)^2\hat s^{\,2}
\notag\\
&\quad+
\frac{1}{2}
\sum_{j=1}^{2l-1}\sum_{k=1}^{K}
\delta\!\left(
\eta\sqrt{2l-1}\,(B+\epsilon)s
\right)
\notag\\
&\le
\frac{8e^4\kappa\|\bm w\|_2^2}{\gamma^2}
(2l-1)^2K(B+\epsilon)^2s^2
\notag\\
&\quad+
\frac{\eta^2}{2}(2l-1)^2K(B+\epsilon)^2s^2
\notag\\
&=
\frac{8(e^4+1)\kappa\|\bm w\|_2^2}{\gamma^2}
(2l-1)^2K(B+\epsilon)^2s^2
\notag\\
&\lesssim
\mathcal O\!\left(
\frac{
l^2K(B+\epsilon)^2s^2\|\bm w\|_2^2
}{
\gamma^2
}
\right),
\label{eq:mpgnn-fixed-beta-KL-bound}
\end{align}
where
\(\eta^2=16\kappa\|\bm w\|_2^2/\gamma^2\),
\(\kappa=1+2\ln 2+\sqrt{4\ln 2}\), 
\[
\delta(x)
\triangleq
\frac{1}{1+x^2}
+
\log(1+x^2)
-
1
\le
x^2.
\]
Recall that
\[
s(\beta)
=
\begin{cases}
\beta^l,
& \tau\neq 1,\\[0.4em]
(l-1)\beta^2,
& \tau=1.
\end{cases}
\]
Thus, the fixed-\(\beta\) KL bound can be written as
\[
D_{\mathrm{KL}}(Q\|P)
\lesssim
\begin{cases}
\mathcal O\!\left(
\dfrac{
l^2K(B+\epsilon)^2\beta^{2l}\|\bm w\|_2^2
}{
\gamma^2
}
\right),
& \tau\neq 1,\\[1.2em]
\mathcal O\!\left(
\dfrac{
l^4K(B+\epsilon)^2\beta^4\|\bm w\|_2^2
}{
\gamma^2
}
\right),
& \tau=1.
\end{cases}
\]

Hence, for any fixed cover point \(\hat\beta\), with probability at least
\(1-\delta\), for any \(\bm w\) such that
\(|\beta-\hat\beta|\le \frac{1}{l+1}\beta\) when \(\tau\neq1\), and
\(|\beta-\hat\beta|\le \frac{1}{3}\beta\) when \(\tau=1\), we have the following bounds, i.e., 
If \(\tau\neq 1\), then
\begin{equation}
\label{eq:mpgnn-fixed-beta-generalization-bound-tau-neq-1}
\begin{aligned}
R_{\mathcal D,0}(f_{\bm w})
&\le
\hat R_{S,\gamma}(f_{\bm w})
\\
&\hspace*{0.8em}+
\mathcal O\!\left(
\sqrt{
\frac{
l^2K(B+\epsilon)^2\beta^{2l}\|\bm w\|_2^2
+
\ln\frac{m}{\delta}
}{
\gamma^2 m
}
}
\right);
\end{aligned}
\end{equation}
If \(\tau=1\), then
\begin{equation}
\label{eq:mpgnn-fixed-beta-generalization-bound-tau-eq-1}
\begin{aligned}
R_{\mathcal D,0}(f_{\bm w})
&\le
\hat R_{S,\gamma}(f_{\bm w})
\\
&\quad+
\mathcal O\!\left(
\sqrt{
\frac{
l^4K(B+\epsilon)^2\beta^4\|\bm w\|_2^2
+
\ln\frac{m}{\delta}
}{
\gamma^2 m
}
}
\right).
\end{aligned}
\end{equation}

Finally, following the same arguments in
\cite{Liao-PACBayes-GNN,Sun-PACBayes-Robust}, we take a union bound over
multiple choices of \(\hat\beta\) so that the bound holds for any
\(\beta\). In what follows, we treat the two cases separately.

\emph{1) If \(\tau\neq1\)}, then the nontrivial values of \(\beta\) satisfy
\begin{equation}
\label{eq:mpgnn-beta-range-tau-neq-1}
\left(
\frac{\gamma}{2(B+\epsilon)}
\right)^{1/l}
\le
\beta
\le
\left(
\frac{\gamma\sqrt m}{2(B+\epsilon)}
\right)^{1/l} .
\end{equation}

On one hand, if \(\beta^l<\gamma/(2(B+\epsilon))\), then for any attacked graph \(G'\) and any \(i,j\in[K]\), 
using the Frobenius-norm bound on the hidden representation in \eqref{eq:app-Hjminus-F-bound}, we have
\begin{equation}
\label{eq:mpgnn-trivial-margin-tau-neq-1}
\begin{aligned}
\left|
RM(f_{\bm w}(G),i,j)
\right|
&=
\left|
M(f_{\bm w}(G'),i,j)
\right|
\le
2\|f_{\bm w}(G')\|_2
\\
&=
2\left\|
\frac{1}{n}\bm 1_n^{T}\bm H_{l-1}'\bm W_l
\right\|_2
\\
&\le
2\left\|
\frac{1}{n}\bm 1_n^{T}
\right\|_2
\|\bm H_{l-1}'\bm W_l\|_F
\\
&\le
\frac{2}{\sqrt n}
\|\bm H_{l-1}'\|_F
\|\bm W_l\|_2
\\
&\le
2(B+\epsilon)
L M_1M_2
\frac{\tau^{l-1}-1}{\tau-1}
\\
&=
2(B+\epsilon)\beta^l
<
\gamma .
\end{aligned}
\end{equation}
Therefore, based on the definition in Eq.~\eqref{eq:robust-margin-losses}, we
always have \(\hat R_{S,\gamma}(f_{\bm w})=1\) when \(\beta^l<\gamma/(2(B+\epsilon))\).

On the other hand, if
\(\beta^l>\gamma\sqrt m/(2(B+\epsilon))\), it follows that
\begin{equation}
\label{eq:mpgnn-large-beta-trivial-tau-neq-1}
\sqrt{
\frac{
l^2K(B+\epsilon)^2\beta^{2l}\|\bm w\|_2^2
+
\ln\frac{m}{\delta}
}{
\gamma^2m
}
}
\ge
\sqrt{
\frac{
l^2K\|\bm w\|_2^2
}{
4
}
}
\ge 1,
\end{equation}
with \(l\ge2\), \(K\ge1\), and
\(\|\bm w\|_2^2\ge 1\). 

To make \(|\beta-\hat\beta|\le \frac{\beta}{l+1}\) satisfied, we require
\(|\beta-\hat\beta|\le \frac{1}{l+1}
\left(\frac{\gamma}{2(B+\epsilon)}\right)^{1/l}\).
If a covering of the interval in Eq.~\eqref{eq:mpgnn-beta-range-tau-neq-1}
with radius \(\frac{1}{l+1}
\left(\frac{\gamma}{2(B+\epsilon)}\right)^{1/l}\)
can make the fixed-\(\hat\beta\) bound validated with \(\hat\beta\)
taking all possible values from the covering, then we could end up with a
bound that holds for all \(\beta\). It follows that such a covering exists
with size at most \((l+1)m^{\frac{1}{2l}}\). Taking a union bound over all
choices of \(\hat\beta\) with the cover yields the final generalization bound with probability 
\(1-\delta\) for any \(\beta\),
\begin{equation}
\label{eq:mpgnn-final-bound-tau-neq-1}
\begin{aligned}
R_{\mathcal D,0}(f_{\bm w})
&\le
\hat R_{S,\gamma}(f_{\bm w})
\\
&\hspace*{-1.0em}+
\mathcal O\!\left(
\sqrt{
\frac{
l^2K(B+\epsilon)^2\beta^{2l}\|\bm w\|_2^2
+
\ln\frac{m(l+1)}{\delta}
}{
\gamma^2m
}
}
\right).
\end{aligned}
\end{equation}

\emph{2) If \(\tau=1\)}, the nontrivial values of
\(\beta\) satisfy
\begin{equation}
\label{eq:mpgnn-beta-range-tau-eq-1}
\left(
\frac{\gamma}{2(B+\epsilon)(l-1)}
\right)^{\frac{1}{2}}
\le
\beta
\le
\left(
\frac{\gamma\sqrt m}{2(B+\epsilon)(l-1)}
\right)^{\frac{1}{2}} .
\end{equation}
If \(\beta^2<\gamma/(2(B+\epsilon)(l-1))\), then \(\hat R_{S,\gamma}(f_{\bm w})=1\); if
\(\beta^2>\gamma\sqrt m/(2(B+\epsilon)(l-1))\), then
\begin{equation}
\label{eq:mpgnn-large-beta-trivial-tau-eq-1}
\sqrt{
\frac{
l^4K(B+\epsilon)^2\beta^4\|\bm w\|_2^2
+
\ln\frac{m}{\delta}
}{
\gamma^2m
}
}
\ge
\sqrt{
\frac{
l^4K\|\bm w\|_2^2
}{
4(l-1)^2
}
}
\ge 1,
\end{equation}
and there exists a covering of the interval in Eq.~\eqref{eq:mpgnn-beta-range-tau-eq-1} with size at most
\(3m^{\frac{1}{4}}\). Taking a union bound over all choices of \(\hat\beta\) with the cover yields the final generalization bound with probability \(1-\delta\) for any \(\beta\), i.e.,
\begin{equation}
\label{eq:mpgnn-final-bound-tau-eq-1}
\begin{aligned}
R_{\mathcal D,0}(f_{\bm w})
&\le
\hat R_{S,\gamma}(f_{\bm w})
\\
&\hspace*{-0.2em}+
\mathcal O\!\left(
\sqrt{
\frac{
l^4K(B+\epsilon)^2\beta^4\|\bm w\|_2^2
+
\ln\frac{m}{\delta}
}{
\gamma^2m
}
}
\right).
\end{aligned}
\end{equation}
Summarizing two cases completes the proof.
\end{proof}

\subsection{Proof of Corollary~\ref{cor:gcn-robust-generalization}}
\label{app:proof-gcn-robust-generalization}
\begin{proof}
For the readout layer,
\[
\frac{\partial f_{\bm w}(G)}
{\partial \mathrm{vec}(\bm H_{l-1})}
=
\frac{1}{n}
\left(
\bm W_l^{T}\otimes \bm 1_n^{T}
\right).
\]
For \(k=j+1,\ldots,l-1\), by the chain rule, we have
\[
\frac{\partial \mathrm{vec}(\bm H_k)}
{\partial \mathrm{vec}(\bm H_{k-1})}
=
\bm B_k^{\phi}
\left(
\bm W_k^{T}\otimes \bm P_G
\right).
\]
Therefore,
\begin{equation}
\label{eq:gcn-Mj}
\bm M_j
=
\frac{1}{n}
\left(
\bm W_l^{T}\otimes \bm 1_n^{T}
\right)
\left(
\prod_{k=l-1}^{j+1}
\bm B_k^{\phi}
\left(
\bm W_k^{T}\otimes \bm P_G
\right)
\right).
\end{equation}
For the message-passing weight matrix \(\bm W_j\), we have
\[
\frac{\partial \mathrm{vec}(\bm H_j)}
{\partial \mathrm{vec}(\bm W_j)}
=
\bm B_j^{\phi}
\left(
\bm I_{h_j}\otimes \bm P_G\bm H_{j-1}
\right),
\]
\begin{equation}
\label{eq:gcn-JWj}
\bm J_{W_j}
=
\bm M_j\bm B_j^{\phi}
\left(
\bm I_{h_j}\otimes \bm P_G\bm H_{j-1}
\right),
\quad j\in[l].
\end{equation}

We next bound the spectral norm of \(\bm J_{W_j}\) as
\begin{equation}
\label{eq:gcn-JWj-spectral-step}
\begin{aligned}
\|\bm J_{W_j}\|_2
&\le
\|\bm M_j\|_2
\|\bm B_j^{\phi}\|_2
\|\bm P_G\bm H_{j-1}\|_2
\\
&\le
L\|\bm M_j\|_2
\|\bm P_G\|_2
\|\bm H_{j-1}\|_F,
\end{aligned}
\end{equation}
where
\begin{equation}
\label{eq:gcn-Mj-spectral-bound}
\|\bm M_j\|_2
\le
\frac{\|\bm W_l\|_2}{\sqrt n}
\prod_{k=j+1}^{l-1}
L\|\bm P_G\|_2\|\bm W_k\|_2 .
\end{equation}
Moreover, for \(k=1,\ldots,j-1\), we have
\begin{equation}
\label{eq:gcn-Hk-recursion}
\|\bm H_k\|_F
\le
L\|\bm P_G\|_2
\|\bm H_{k-1}\|_F
\|\bm W_k\|_2 .
\end{equation}
Iterating \eqref{eq:gcn-Hk-recursion} with \(\bm H_0=\bm X\) gives
\begin{equation}
\label{eq:gcn-Hj-bound}
\begin{aligned}
\|\bm H_{j-1}\|_F
&\le
\|\bm X\|_F
\prod_{k=1}^{j-1}
L\|\bm P_G\|_2\|\bm W_k\|_2
\\
&\le
\sqrt n B
\prod_{k=1}^{j-1}
L\|\bm P_G\|_2\|\bm W_k\|_2 .
\end{aligned}
\end{equation}
Combining \eqref{eq:gcn-JWj-spectral-step},
\eqref{eq:gcn-Mj-spectral-bound}, and \eqref{eq:gcn-Hj-bound}, we obtain
\begin{equation}
\label{eq:gcn-JWj-spectral-bound}
\|\bm J_{W_j}\|_2
\le
B
\left(L\|\bm P_G\|_2\right)^{l-1}
\prod_{\substack{k=1\\ k\neq j}}^{l}
\|\bm W_k\|_2.
\end{equation}
Since \(\phi_j\) is the ReLU activation function, we have \(L=1\),
and we use the normalized adjacency matrix \(\|\bm P_G\|_2\le1\), such that
\begin{equation}
\label{eq:gcn-JWj-spectral-bound-simplified}
\|\bm J_{W_j}\|_2
\le
B
\prod_{k\neq j}
\|\bm W_k\|_2 .
\end{equation}
Let
\[
\beta
=
\left(
\prod_{k=1}^{l}\|\bm W_k\|_2
\right)^{1/l}.
\]
By the homogeneity of GCNs, we normalize the weights as
\[
\tilde{\bm W}_k
=
\frac{\beta}{\|\bm W_k\|_2}\bm W_k,
\]
so that we have \(f_{\bm w}=f_{\tilde{\bm w}}\). For the normalized weights,
\(\|\tilde{\bm W}_k\|_2=\beta\), then
\eqref{eq:gcn-JWj-spectral-bound-simplified} gives
\begin{equation}
\label{eq:gcn-JWj-beta-bound}
\|\bm J_{W_j}\|_2
\le
B\beta^{l-1}.
\end{equation}

Since \(f_{\bm w}(G)\in\mathbb R^K\), the Jacobian matrices of GCN output have rank at most \(K\). Hence, for \(j\in[l]\), we have
\begin{equation}
\label{eq:gcn-JWj-svd}
\begin{aligned}
\bm J_{W_j}
&=
\bm Q_{W_j}\bm S_{W_j}\bm V_{W_j}^{T},
\\
\bm S_{W_j}
&=
\diag(s_{W_j,1},\ldots,s_{W_j,K},0,\ldots,0),
\end{aligned}
\end{equation}
where \(s_{W_j,1}\ge\cdots\ge s_{W_j,K}\ge0\), and we have
\(s_{W_j,r}\le B\beta^{l-1}\) for \(r\in[K]\). Accordingly, define
\begin{equation}
\label{eq:gcn-AWj-def}
\bm A_{W_j}
=
\sqrt l\,
\bm V_{W_j}
\diag(
\underbrace{B\beta^{l-1},\ldots,B\beta^{l-1}}_{K},
0,\ldots,0
)
\bm V_{W_j}^{T}.
\end{equation}
It is clear that
\begin{equation}
\label{eq:gcn-JWj-AWj-domination}
\bm J_{W_j}^{T}\bm J_{W_j}
\preceq
\frac{1}{l}\bm A_{W_j}^{T}\bm A_{W_j},
\quad j\in[l].
\end{equation}
Therefore, for any perturbation \(\bm u\), we have
\begin{equation}
\label{eq:gcn-lr-sensitivity-domination}
\begin{aligned}
\|f_{\bm w+\bm u}(G)-f_{\bm w}(G)\|_2^2
&=
\left\|
\sum_{j=1}^{l}\bm J_{W_j}\bm u_{W_j}
\right\|_2^2
\\
&\le
\sum_{j=1}^{l}
\|\bm A_{W_j}\bm u_{W_j}\|_2^2 .
\end{aligned}
\end{equation}

Since the prior cannot depend on the learned value of \(\beta\), we fix a cover point \(\hat\beta\) such that
\(|\beta-\hat\beta|\le \frac{\beta}{l}\).
Then \(\beta^{l-1}\le e\hat\beta^{l-1}\). Let us introduce the approximations of the sensitivity matrices
\begin{equation}
\label{eq:gcn-AWj-hat-def}
\hat{\bm A}_{W_j}
=
\sqrt l\,
\bm V_{W_j}
\diag(
\underbrace{B\hat\beta^{l-1},\ldots,B\hat\beta^{l-1}}_{K},
0,\ldots,0
)
\bm V_{W_j}^{T},
\end{equation}
and optimize the posterior covariance to minimize the KL with
\[
\bm R_{W_j}
=
\left(
\bm I+\eta^2
\hat{\bm A}_{W_j}^{T}\hat{\bm A}_{W_j}
\right)^{-1},
\quad
\eta^2=
\frac{16\kappa\|\bm w\|_2^2}{\gamma^2}.
\]
Since \(\bm R_{W_j}\preceq \bm I\), by the Gaussian concentration
bound, with probability at least \(1/2\) over \(\bm u\), we have
\begin{equation}
\label{eq:gcn-gaussian-concentration}
\begin{aligned}
\sum_{j=1}^{l}
\|\bm A_{W_j}\bm u_{W_j}\|_2^2
&\le
\sigma^2\kappa
\sum_{j=1}^{l}
\Tr\!\left(
\bm A_{W_j}
\bm R_{W_j}
\bm A_{W_j}^{T}
\right)
\\
&\le
e^2\sigma^2\kappa
\sum_{j=1}^{l}
\Tr\!\left(
\hat{\bm A}_{W_j}
\hat{\bm A}_{W_j}^{T}
\right)
\\
&=
e^2\sigma^2\kappa l^2K B^2\hat\beta^{2l-2}.
\end{aligned}
\end{equation}

Combining \eqref{eq:gcn-gaussian-concentration} with
\eqref{eq:gcn-lr-sensitivity-domination}, we obtain
\begin{equation}
\label{eq:gcn-output-perturbation}
\|f_{\bm w+\bm u}(G)-f_{\bm w}(G)\|_2^2
\le
e^2\sigma^2\kappa l^2K B^2\hat\beta^{2l-2}.
\end{equation}

Consider the \(\epsilon\)-attack with \(\epsilon>0\), we have
\begin{equation}
\label{eq:gcn-robust-margin-perturbation}
\begin{aligned}
\MoveEqLeft \left|
RM(f_{\bm w+\bm u}(G),i,j)
-
RM(f_{\bm w}(G),i,j)
\right|^2
\\
&\le
4e^2\sigma^2\kappa l^2K(B+\epsilon)^2\hat\beta^{2l-2}.
\end{aligned}
\end{equation}
By setting
\begin{equation}
\label{eq:gcn-sigma-choice}
\frac{1}{\sigma^2}
=
\frac{
16e^2\kappa l^2K(B+\epsilon)^2\hat\beta^{2l-2}
}{
\gamma^2
},
\end{equation}
the robust perturbation condition
\[
\left|
RM(f_{\bm w+\bm u}(G),i,j)
-
RM(f_{\bm w}(G),i,j)
\right|^2
\le
\frac{\gamma^2}{4}
\]
is satisfied with probability at least \(1/2\) over \(\bm u\).

It remains to derive the KL term and remove the dependence of the prior
on the learned value \(\beta\). Plugging the choice of
\(1/\sigma^2\) in \eqref{eq:gcn-sigma-choice} into the KL objective gives
\begin{equation}
\label{eq:gcn-KL-bound}
D_{\mathrm{KL}}(Q\|P)
\lesssim
\mathcal O\!\left(
\frac{
l^2K(B+\epsilon)^2\beta^{2l-2}\|\bm w\|_2^2
}{
\gamma^2
}
\right).
\end{equation}
Therefore, for any fixed cover point \(\hat\beta\), we have
\begin{equation}
\label{eq:gcn-fixed-beta-bound}
\begin{aligned}
R_{\mathcal D,0}(f_{\bm w})
&\le
\hat R_{S,\gamma}(f_{\bm w})
\\
&\hspace*{-1.0em}+
\mathcal O\!\left(
\sqrt{
\frac{
l^2K(B+\epsilon)^2\beta^{2l-2}\|\bm w\|_2^2
+
\ln\frac{m}{\delta}
}{
\gamma^2m
}
}
\right).
\end{aligned}
\end{equation}

With the same arguments as in the MPGNN, we end up with the final robust generalization bound
\begin{equation}
\label{eq:gcn-final-bound}
\begin{aligned}
R_{\mathcal D,0}(f_{\bm w})
&\le
\hat R_{S,\gamma}(f_{\bm w})
\\
&\hspace*{-1.0em}+
\mathcal O\!\left(
\sqrt{
\frac{
l^2K(B+\epsilon)^2\beta^{2l-2}\|\bm w\|_2^2
+
\ln\frac{ml}{\delta}
}{
\gamma^2m
}
}
\right).
\end{aligned}
\end{equation}
This completes the proof.
\end{proof}

\bibliographystyle{IEEEtran}
\bibliography{references}

\end{document}